\documentclass[10pt]{article}
\usepackage{amsmath, amsthm, amsfonts}
\usepackage{wrapfig}

\usepackage{makecell}
\usepackage{tabularx}
\usepackage{float}
\usepackage{afterpage}
\usepackage{xcolor}
\usepackage{graphicx}
\usepackage{caption}
\usepackage{booktabs}
\usepackage[numbers]{natbib}

\newtheorem{theorem}{Theorem}
\newtheorem{lemma}{Lemma}
\definecolor{cvprblue}{rgb}{0.21,0.49,0.74}

\usepackage{indentfirst,csquotes}

\usepackage[pagebackref,breaklinks,colorlinks,allcolors=cvprblue]{hyperref}

\title{Oscillation Inversion: Understand the structure of Large Flow Model through the Lens of Inversion Method}

    \author{Yan Zheng$^{1,*}$, Zhenxiao Liang$^{1,*}$, Xiaoyan Cong$^2$,\\
Lanqing guo$^1$, Yuehao Wang$^1$, Peihao Wang$^1$, Zhangyang Wang$^1$\\
[6pt]
$^1$ University of Texas at Austin
$^2$ Brown University
}

\begin{document}
\maketitle
\begin{center}
    \centering
    \captionsetup{type=figure}
    \includegraphics[width=1.0\textwidth]{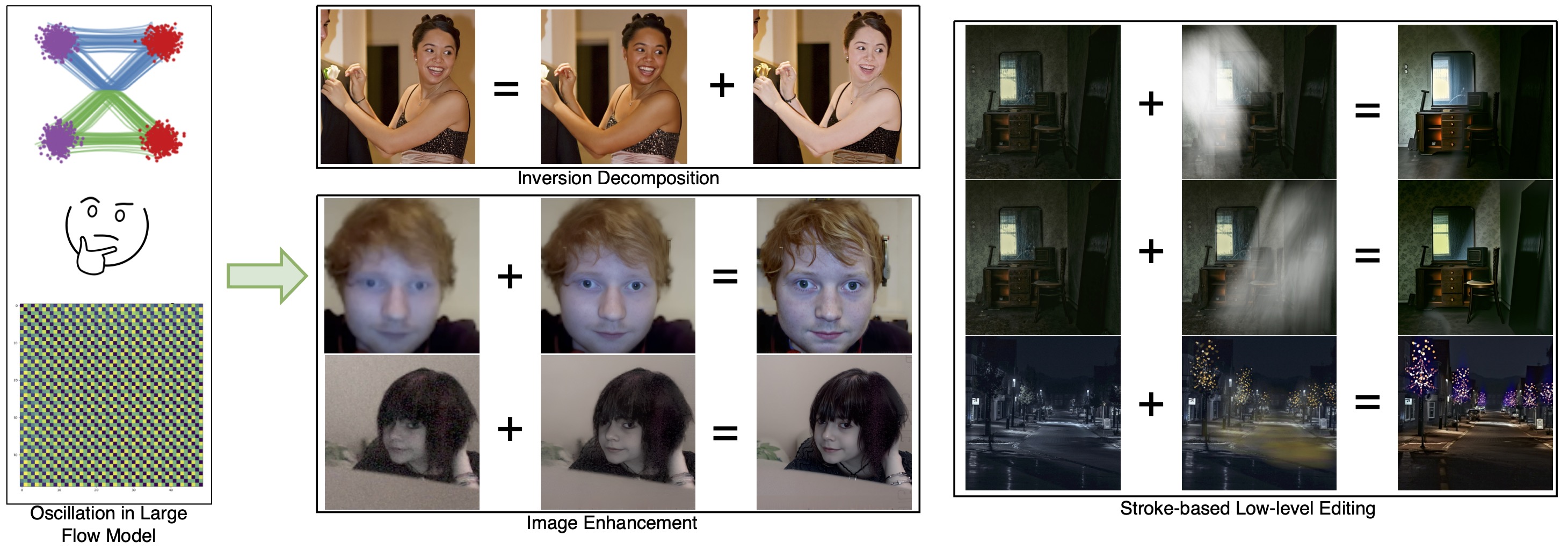}
    
    \captionof{figure}{Oscillation Inversion  is a phenomenon observed in large flow models. Building on this, we developed a simple and fast method that serves as a distribution transfer technique, enabling image enhancement as well as low-level editing, e.g. stroke-based relighting and recoloring.}
\label{fig:teaser}
\end{center}

\renewcommand{\thefootnote}{}
\footnotetext{$^*$ Equal Contribution}
\renewcommand{\thefootnote}{\arabic{footnote}}

\begin{abstract}
We explore the oscillatory behavior observed in inversion methods applied to large-scale text-to-image diffusion models, with a focus on the "Flux" model. By employing a fixed-point-inspired iterative approach to invert real-world images, we observe that the solution does not achieve convergence, instead oscillating between distinct clusters. Through both toy experiments and real-world diffusion models, we demonstrate that these oscillating clusters exhibit notable semantic coherence. We offer theoretical insights, showing that this behavior arises from oscillatory dynamics in rectified flow models. Building on this understanding, we introduce a simple and fast distribution transfer technique that facilitates image enhancement, stroke-based recoloring, as well as visual prompt-guided image editing. Furthermore, we provide quantitative results demonstrating the effectiveness of our method for tasks such as image enhancement, makeup transfer, reconstruction quality, and guided sampling quality. Higher-quality examples of videos and images are available at \href{https://yanyanzheng96.github.io/oscillation_inversion/}{this link}.
\end{abstract}

\section{Introduction}

{Very recently, large text-to-image diffusion models utilizing rectified flow~\cite{wang2024_flowtarget}, like the Flux model from Black Forest~\cite{blackforestlabs_flux}, have demonstrated exceptional performance in generating high-quality images with rapid sampling.
However, the underlying latent structure of rectified flow-based models presents unique challenges, as it differs fundamentally from the layered manifold structure of DDPMs~\cite{ho2020denoising} shaped by a parameterized Markov chain.
This distinction makes previous inversion techniques, such as DDIM inversion~\cite{song2020denoising}, and editing methods like SDEdit~\cite{meng2021sdedit}, less viable. 
Therefore, adopting a new perspective for understanding and navigating the latent space of these flow-based models is essential for enabling more effective inversion and image manipulation strategies.
}

{
When attempting to invert real-world images using fix-point iteration methods in rectified flow-based models, we observe that the sequence of iterates does not converge to a single point but instead oscillates between several clusters.
These clusters are semantically meaningful and can be leveraged for image optimization and editing tasks.
This behavior contrasts with the fixed-point methods used in DDIM\cite{aidi, garibi2024renoise}, which are primarily designed to mitigate numerical accumulation errors at each step, ensuring smooth convergence to a single, stable solution. 
%
The oscillatory nature observed in rectified flow-based models, however, opens up new possibilities for iterative refinement and enhanced flexibility in image inversion tasks.
}

{
To investigate this, we first propose \textbf{Oscillation Inversion}, a method that uses fixed-point iteration to directly establish a one-to-one mapping between noisy latents at a given timestep and the corresponding encoded image latent.
The inverted latents oscillate among several clusters, which can serve as local latent distributions, facilitating effective semantic-based image optimization.
Additionally, we generalize this fixed-point method in three ways for broader downstream applications:
\textbf{1) Group Inversion}: We invert a group of images simultaneously, rather than a single image, enabling semantic guidance and blending across images.
\textbf{2) Finetuned Inversion}: By controlling the oscillation direction, we provide a  mechanism for customized, user-driven manipulation of the images.
\textbf{3) Post-Inversion Optimization}: After inversion, we perform optimization, and analyze the differentiable structure induced by the separated sub-distributions created by oscillation inversion.
%
These extensions make Oscillation Inversion a versatile tool for various image manipulation tasks.
}

{
The main contributions of this work are as follows:
\begin{itemize}
    \item We propose Oscillation Inversion to facilitate one-step inversion to any timestep in rectified flow-based diffusion models for the diverse manipulation of latents. Additionally, we present three extensions that enable diverse user inputs for real-world applications.
    \item We propose a theoretical framework to explain and validate the oscillation phenomenon, which aligns closely with the results obtained from rectified flow trained on a toy dataset.
    \item Extensive experiments on various downstream tasks, such as image restoration and enhancement, stroke based make up transfer, validate our theoretical findings and demonstrate the effectiveness of our method on both perceptual quality and data fidelity.
\end{itemize}
}

\section{Related Works}
\label{related_works}

\noindent\textbf{Flow Model. } 
Diffusion models~\cite{rombach2022high, saharia2022photorealistic, ramesh2022hierarchical} generate data by a stochastic differential equation (SDE)-based denoising process and probability flow ordinary differential equations (ODE)~\cite{song2020score, lipman2022flow, lipman2022flow, salimans2022progressive, song2023consistency} improves sampling efficiency by formulating the denoising process into a ODE-based process. However, probability flow ODE-based methods  suffer from the computational expense of denoising via numerical integration with small step sizes. To address these issues, some simulation-free flow models have emerged, e.g. flow matching~\cite{lipman2022flow} and rectified flow~\cite{liu2022flow}. Flow matching introduces a training objective for continuous normalizing flows~\cite{chen2018neural} to regress the vector field of a probability path. Rectified flow learns a transport map between two distributions through constraining the ODE to follow the straight transport paths. Since the latent structure of flow models differs fundamentally from the layered manifold structure of Denoising Diffusion Probabilistic Models (DDPMs)~\cite{ho2020denoising}, it is valuable to explore the intrinsic characteristics of the flow models' latent space.

\vspace{10pt}
\noindent\textbf{Diffusion-based Inversion.} The rise of diffusion models~\cite{rombach2022high, saharia2022photorealistic, ramesh2022hierarchical} has unlocked significant potential of inversion methods for real image editing, which are primarily categorized into DDPM~\cite{ho2020denoising} based~\cite{wu2023latent, hubermanspiegelglas2024editfriendlyddpmnoise} and DDIM~\cite{song2020denoising} based methods~\cite{pan2023effective, garibi2024renoise, li2024source, meiri2023fixed}. While DDPM-based methods yield impressive editing results, they are hindered by their inherently time-consuming and stochastic nature, due to the random noise introduced across a large number of inversion steps~\cite{wu2023latent, hubermanspiegelglas2024editfriendlyddpmnoise}. DDIM-based methods utilize the DDIM sampling strategy to enable a more deterministic inversion process, substantially reducing computational overhead and time. However, the linear approximation behind DDIM often leads to error propagation, resulting in reconstruction inaccuracy and content loss, especially when classifier-free guidance (CFG) is applied~\cite{mokady2023null}. Recent approaches,~\cite{wallace2023edict, mokady2023null, pan2023effective, miyake2023negative, han2023improving, Hong_2024_CVPR}, address these issues by aligning the diffusion and reverse diffusion trajectories through the optimization of null-text tokens~\cite{mokady2023null} or prompt embeddings~\cite{han2023improving, miyake2023negative}. EDICT~\cite{wallace2023edict} and BDIA~\cite{zhang2023exact} introduce invertible neural network layers to enhance computational efficiency and inversion accuracy, though these methods suffer from notably longer inversion times. To tackle this, recent works~\cite{meiri2023fixed, pan2023effective}~\cite{garibi2024renoise, li2024source} have adopted fixed-point iteration for each inversion step, mitigating numerical error accumulation and ensuring smooth convergence to a single, stable solution. Interestingly, when applied to rectified flow-based methods, the sequence of fixed-point iterates oscillates between several semantically meaningful clusters, presenting significant potential for downstream applications.
\begin{figure*}[t]
  \centering
  \includegraphics[width=\textwidth]{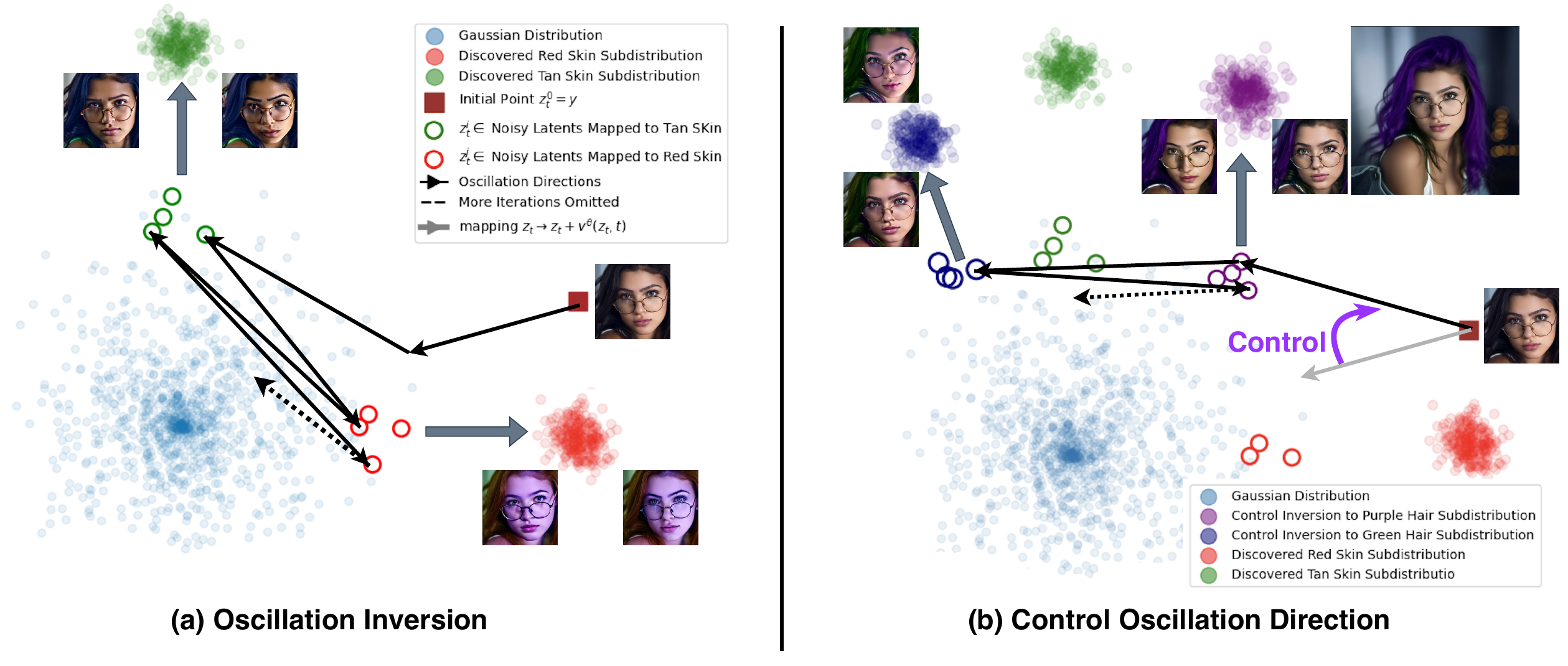}
  \caption{In the left figure (a), fixed-point iteration causes oscillation, leading to subdomains with opposite features in the case of the brown-skinned girl, resulting in more tan and red tones. In the right figure (b), we demonstrate how this oscillation can be customized to achieve desired separations, such as changes in hair color. }
  \label{fig:pipeline}
\end{figure*}

\section{Oscillation Inversion}
\label{method}

\subsection{Preliminary}

\subsubsection{Rectified Flow}

Rectified flow~\cite{liu2022flow} is a novel generative approach that facilitates smooth transitions between two distributions, denoted $\pi_0$ for noise and $\pi_1$ for target, by solving ordinary differential equations (ODEs). Specifically, for $X_0 \sim \pi_0$ and $X_1 \sim \pi_1$, the transition between $x_0$ and $x_1$ is defined through an interpolation given by $X_t = (1-t)X_0 + tX_1$ for $t \in [0, 1]$. \cite{liu2022flow} demonstrated that, starting from $Z_0 \sim \pi_0$, the following ODE can be used to obtain a trajectory that preserves the marginal distribution of $Z_t$ at any given time $t$:
\begin{equation}
    \frac{d Z_t}{dt} = v^{X}(Z_t,t),\ \textup{where } v^{X}(x,t) := \mathbb{E}[X_1 - X_0 \mid X_t = x].
\label{eq: rectified flow}
\end{equation}

The solution of $v^{X}$ in Eq.~(\ref{eq: rectified flow}) is obtained by optimizing the following loss via stochastic coupling sampling $(X_0,X_1)\sim (\pi_0,\pi_1)$ and $t\sim \textup{Uniform}([0,1])$,
\begin{equation}
    v^{X} = \underset{v}{\arg\min}\ \mathbb{E} \left[ \big\| (X_1 - X_0) - v(X_t, t) \, \big\|^2 \right].
\label{eq: optimization formula of rectified flow}
\end{equation}

\subsection{Method}


In this section, we first formulate the inversion problem for rectified flow-based models (Sec.\ref{sec:problem}). To address this, we introduce Oscillation Inversion, which constructs one-step inversion using fixed-point iteration (Sec.\ref{sec:oscillation}). We then propose fine-tuned inversion to enable controllable latent structures (Sec.\ref{sec:finetune}), enabing make-up transfer task in Sec.\ref{section:makeup} , followed by post-inversion optimization that leverages the differentiation structure introduced by one-step inversion (Sec.\ref{sec:postoptim}), enabling visual prompt-guided sampling(Sec.\ref{sec:optim_app}). The concept behind our general method is depicted in Figure \ref{fig:pipeline}.


\subsubsection{Inversion Problem}\label{sec:problem}
In practice, the large flow model in the context of generative modeling operates within the latent space of a Variational Autoencoder (VAE)~\cite{kingma2013auto}, utilizing an encoder $E: \mathbb{R}^d \rightarrow \mathbb{R}^n$ and a decoder $D: \mathbb{R}^n \rightarrow \mathbb{R}^d$.
The sampling process begins from Gaussian noise \( z_T \sim \mathcal{N}(0, \mathbf{I}) \), and the latent variable is progressively refined through a sequence of transformations. The forward generative process is defined by the following iterative formula starting from $t=T$ all the way back to $t=0$:
\begin{equation}
    z_{t-1} = z_t + \left( \sigma_{t-1} - \sigma_t \right) v_{\theta}(z_t, \sigma_t),
\end{equation}
where $v_{\theta}(z_t, \sigma_t)$ represents the learned velocity field parameterized by a transformer with weights $\theta$, and $\sigma_t$ is a monotonically increasing time step scaling function depending on time $t$ with $\sigma_0=0$ and $\sigma_T=1$. Here, $T$ denotes the total number of discredited timesteps. The final latent variable, $z_0$, is the output ready for decoding. 

The inversion problem involves seeking for the initial noise $z_T$ given an observed pixel image $I$ with corresponding latent encoding $y \in \mathbb{R}^d$, such that generating from $z_T$ using the flow model described above allows us to either reconstruct $y$ or apply desired modifications to it.

However, the gradual process of sampling from Gaussian noise to the original $y$ diminishes the advantage of GAN-like one-step mappings for direct latent space optimization. To address this, unlike tackling with the initial noise at $t=T$, we introduce the assumption that, at a selected intermediate timestep $t_0$, there exists a direct one-step mapping from the noisy latent at timestep $t_0$ to the clean latent at timestep $0$ via a ``jumping'' transformation.

More specifically, assuming $y$ is the latent code to recover, we aim to figure out the \textit{intermediate} latent code $z_{t_0}$ satisfying
\begin{equation}
    z_{t_0} + \left( \sigma_0 - \sigma_{t_0} \right) v_{\theta}(z_{t_0}, \sigma_{t_0})= y.
    \label{eq: jump}
\end{equation}
Solving Eq.~(\ref{eq: jump}) is non-trivial, and in the following sections, we will describe how we find a set of approximated solutions using iterative method and analyze its oscillating properties.

\subsubsection{Oscillations Inversion}\label{sec:oscillation}

\begin{figure}
\centering
\includegraphics[width=0.95\textwidth]{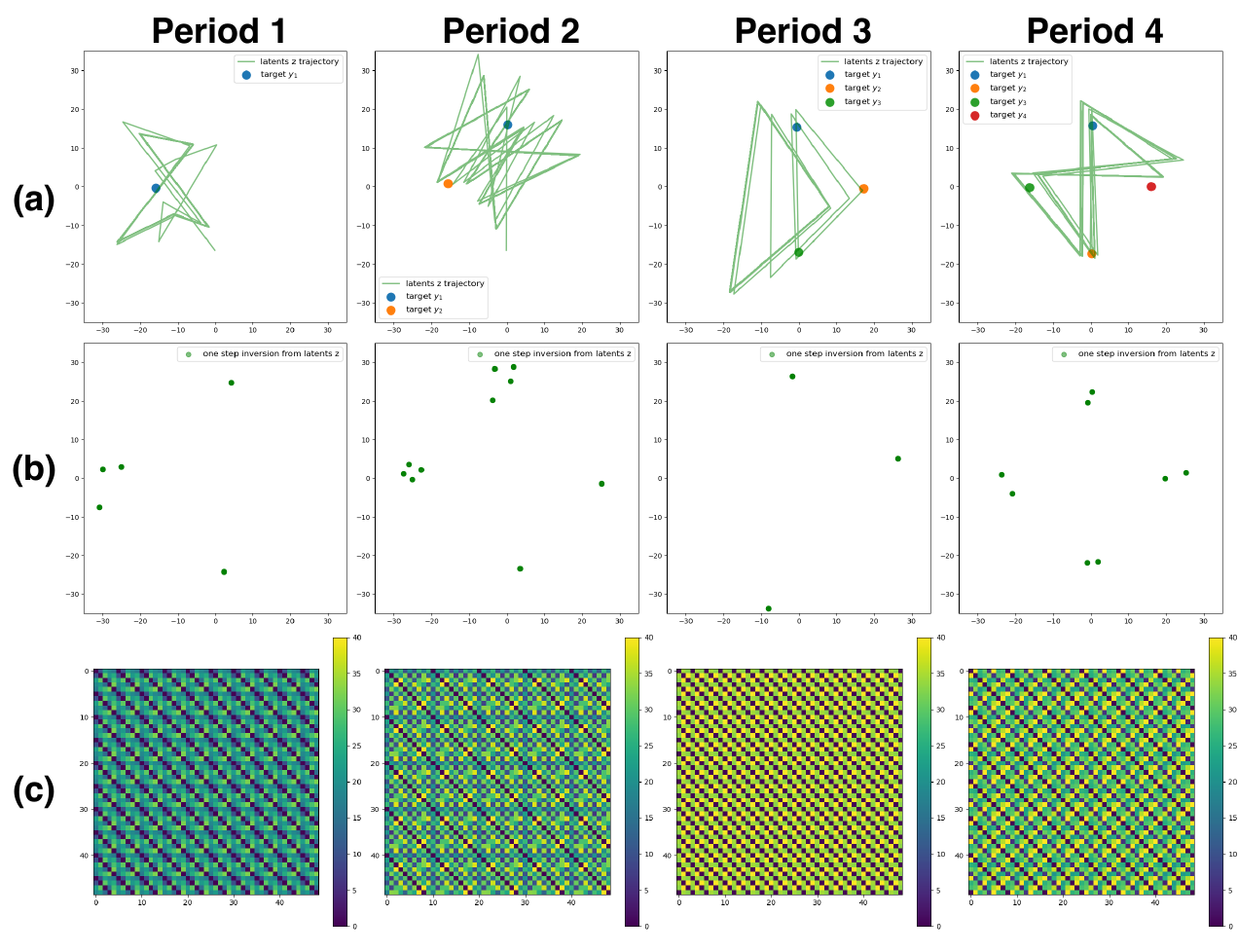}
\caption{By training flow matching on the distribution displayed in Fig. \ref{fig:mixturegaussian}, we demonstrate that the oscillation inversion phenomenon observed in large flow models aligns well with that seen in toy data. The period $m$ indicates that our group inversion starts with $m$ initial $y$s. Row (a) shows the trajectory of inverted latents, $z_{t_0}^{(k+1)}$. Row (b) shows the one-step prediction back to the input space from these latents. Row (c) presents the quantitative distance along the latents' trajectory, reflecting clear periodic patterns.}
\label{fig:theory_refine}
\end{figure}


To address the inversion problem, we employ a fixed-point iteration method to approach the solution of Eq.~(\ref{eq: jump}). Instead of directly seeking a point $z_{t_0}$ such that applying the one-step generative process as described in the left side of (\ref{eq: jump}) yields the target latent $y$, we define an iterative process that refines our approximation of the inverted latent code.
We define the fixed-point iteration as:
\begin{equation} z^{(k+1)}_{t_0} = y - (\sigma_0 - \sigma_{t_0}) v_{\theta}(z^{(k)}_{t_0}, \sigma_{t_0}), 
\label{eq: iterative formula}
\end{equation}
with the initial condition $z^{(0)}_{t_0} = y$.
The sequence $\{ z^{(k)}_{t_0} \}_{k=0}^{\infty}$ represents successive approximations of the inverted latent code at timestep $t$.

As shwon in Figure~\ref{fig:theory_refine}, rather than converging to a single point as suggested by Banach's Fixed-Point Theorem~\cite{banach1922operations}, we empirically observed that the sequence $\{ z^{(k)}_{t_0} \}_{k=0}^{\infty}$ generally oscillates among several clusters in the latent space. Each cluster corresponds to a semantically concentrated region that shares similar low-level features. This oscillatory behavior can be harnessed to explore different variations of the input image, providing a richer inversion that captures multiple aspects of the data.

\begin{figure}[t]
\centering
\includegraphics[width=0.7\textwidth]{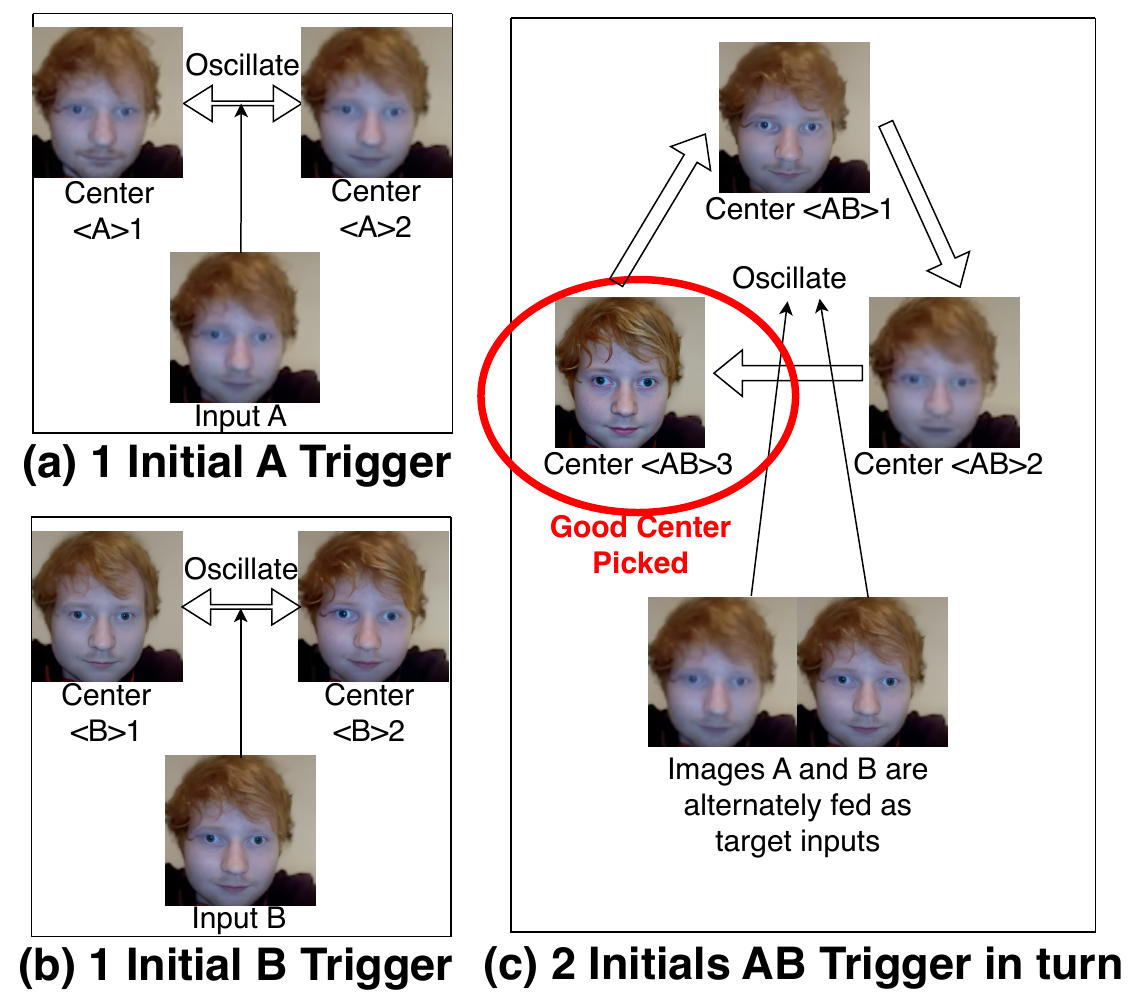}
\caption{This is an example of group inversion, where the high-quality distribution is triggered by two degenerate distributions. We also apply this method to large-scale experiments in Section \ref{sec:largescale}.}
\label{fig:trigger}
\end{figure}

Building upon the fixed-point method, we introduce \textbf{group inversion} that can trigger more stable oscillation phenomena by inverting a set of images simultaneously in a periodic fashion as shown in Figure~\ref{fig:trigger}. Suppose we obtain their corresponding latent encodings $y_1,\dots, y_m$ from  a collection of images $I_1,\dots, I_m$ using the VAE encoding. We could perform the iteration on the group:
\begin{equation}
z_{t_0}^{(k+1)} = y_{(k\bmod m)} - (\sigma_0 - \sigma_{t_0}) v_{\theta}(z_{t_0}^{(k)}, \sigma_{t_0}),
\label{eq:group inversion}
\end{equation}
with initial conditions $z_{t_0}^{(0)} = y_{(1 \bmod m)}$. By inverting the images together, we enable interactions between their latent representations during the iteration process.
Our experimental findings indicate that this collective inversion induces more diverse oscillatory clusters in the latent space. Also, the oscillation phenomenon observed in the flow model trained on a toy distribution transitions from a large central Gaussian to a mixture of four smaller Gaussians(Fig. \ref{fig:mixturegaussian}), as showed in Figure \ref{fig:theory_refine} This behavior aligns closely with the results from our experiments on larger models. As illustrated in Figure \ref{fig:trigger} , we find that group inversion can be used as a domain transfer technique to obtain higher-quality image clusters by mixing two lower-quality clusters. Based on this, we also validate this technique on a larger-scale experiment in Section \ref{sec:largescale}.

\subsubsection{Fintuned Inversion}\label{sec:finetune}

While the fixed-point iteration method introduced earlier reveals oscillations that segment the latent space into separate clusters, these clusters are not directly controllable.

To overcome this limitation, we propose \textbf{finetuned inversion}, a simple fine-tuning step that adjusts the inversion direction, allowing the separated clusters to align with customized semantics. This approach can also induce more diverse oscillatory separations in the latent space.

Given an input image $I$ with its encoded latent representation $y$, we consider the image $\tilde{I}$ modified from $I$ that reflects desired editing, such as alterations made using off-the-shelf masking and inpainting models or customized doodling edits. For example, if original $I$ is an image of a girl with black hair, $\tilde{I}$ could be a roughly edited version where the girl's hair is doodled to appear purple. Importantly, the edited image $\tilde{I}$ does not have to be photo-realistic nor perfect.

Encoding the edited image $\tilde{I}$ to obtain $\tilde{y}$, our goal is to fine-tune the parameters $\theta$ of the velocity field network $v_{\theta}$ so that the inversion process aligns with the desired modifications. The fine-tuning optimization problem could be formulated as $ \underset{\theta}{\textup{minimize}} \mathcal{L}_{\text{finetune}} = \left\| v_{\theta}(y, \sigma_t) - v^{\textup{gt}}(\tilde{y}, \sigma_t) \right\|_2^2 $, where $v_{\theta}$ is the velocity field parameterized by $\theta$ that we aim to finetune, and $v^{\textup{gt}}$ is the pretrained velocity field with frozen weights. An example of visual results can be seen in Figure \ref{fig:gradual}.


\begin{figure}
\centering
\includegraphics[width=0.7\textwidth]{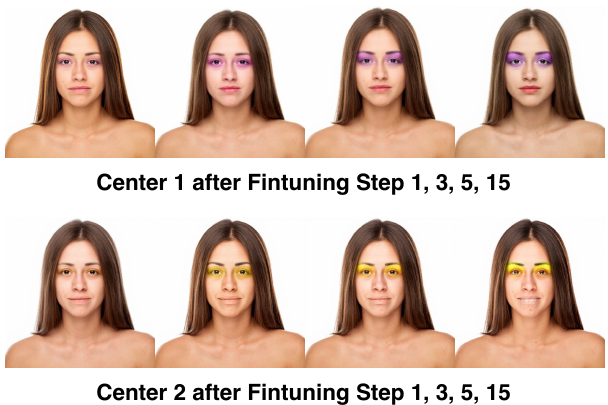}
\caption{Gradual evolution of two oscillating distributions occurs under the influence of visual prompts during fine-tuning. We utilize this method in the experiments discussed in Section \ref{section:makeup}. }
\label{fig:gradual}
\end{figure}

Through force aligning $v_{\theta}(y)$ and $v^{\textup{gt}}(\tilde{y})$ via optimizing over $\theta$, we effectively put strong guidance on the inversion direction so that the resulting oscillatory clusters obtained as described in Sec.~\ref{sec:oscillation} would be pulled towards the semantics specified by the reference image $\tilde{I}$.

\subsubsection{Post-inversion Optimization}\label{sec:postoptim}

Following the fixed-point iterations and optional fine-tuning, we perform a post-inversion optimization to further refine the identified latent code $z_{t_0}$. In Theorem \ref{theorem_mean}, we theoretically demonstrate that the mean of the formulated clusters serves as an approximation for a one-step exact inversion from a single input image. Practically, given an input image, we find that optimization can be performed using either the averaged center among clusters or the individual cluster centers. The former directly facilitates visual guidance, while the latter incorporates cluster regularization. Due to space constraints, we focus on the averaged approach in this section and refer readers to the appendix for details on the cluster-based method.

The optimization problem is formulated as:
\begin{equation}
    \underset{z_{t_0}}{\textup{min}} \ \mathcal{L}_{\textup{rgb}}\circ D \big(z_{t_0} + (\sigma_0 - \sigma_{t_0}) v_{\theta}(z_{t_0}, \sigma_{t_0})\big),
\end{equation}

where $\circ$ represents function composition operator while $D(\cdot)$ is the VAE decoder that converts the latent code back to the pixel space, $\mathcal{L}_{\text{rgb}}$ is a customized loss function defined directly on the decode pixel image space designed based on the specific application or desired attributes. In Sec \ref{sec:optim_app}, we show that this post-one-step inversion optimization can be used for image interpolation and even works on mesh surfaces with underlying geometric constraints.

\section{Analysis}
\label{theory}

In this section, consider a simplified mixed-Gaussian setup trained with an ideal model within rectified flow framework to explain the oscillatory behavior observed in the empirical experiments of fixed-point inversion. We aim to show that, under certain reasonable assumptions, 1. there does not exist any \textit{stable} fixed point solution $z_{t_0}$ for Eq.~\ref{eq: jump}, 2. the iterative formula Eq.~\ref{eq: iterative formula} would \textit{randomly} oscillate between the clusters determined by the target Gaussian mixture distribution, 3. the group inversion introduced in the Section \ref{sec:oscillation} generalizes the concept of a non-converging fixed point, resulting in a periodic dynamic system with multiple solutions within this framework.

\subsection{Model Assumptions}
To align with the notation of the ``Flux'' model's time schedule and the notation introduced in Section~\ref{sec:problem}, we use a flipped notation here: we denote the source domain as $\pi_1$ and the target domain, which the generation is heading towards, as $\pi_0$. The pure noise distribution $\pi_1$ is assumed to be the $d$-dimensional standard normal distribution, and the target distribution $\pi_0$ is a Gaussian mixture distribution given by components $\mathcal{N}(\mu_c,\Sigma_c)$ and corresponding coefficient $\phi_c$ for $c=1,\dots, c_0$. Equivalently, the PDF of $\pi_0$ is determined by

$$
\mathbf{P}_{\pi_0} = \sum_{c=1}^{c_0} \phi_c\cdot \mathbf{P}_{\mathcal{N}(\mu_c,\Sigma_c)}.
$$

Moreover, we consider the ideal case where $v^X$ is the precise solution to Eq.~\ref{eq: optimization formula of rectified flow}, which can be written as

\begin{equation}
    v^{X}(x,t)=\mathbb{E}[X_0-X_1\mid (1-t)X_1+tX_0=x]
\end{equation}

as discussed in \cite{liu2022flow}. Here $X_t\sim\pi_t$ for $0\leq t\leq 1$, while the derived random variable $X_t=tX_1+(1-t)X_0$ is subject to another Gaussian mixture distribution. 

Let $\gamma=t_0$ be the intermediate timestamp that we are interested in. The inversion problem Eq.~\ref{eq: jump} is equivalent to figuring out the fixed points $z$ of function $f(z;y,\gamma)$, which is defined by
\begin{align}
    f(z;y,\gamma) &:= y - (\sigma_0 - \sigma_{\gamma}) v_{\theta}(z, \sigma_{\gamma}) \nonumber\\
    &= y + \gamma v_{\theta}(z, \sigma_{\gamma})
    ,
    \label{eq: formula of f}
\end{align}
in which we assume $\sigma(t)$ is set to trivial $t$ without losing generality. In the rest of this section we may write $f(z)$ with ignoring $y$ and $\gamma$ whenever the context is clear.

By fixed point method, we could seek fixed points of Eq.~\ref{eq: formula of f} with iterative formula
\begin{equation}
z^{(k+1)}:=f(z^{(k)})
\label{eq: def of zk}
\end{equation}
with some initial $z^{(0)}$.

\subsection{Analysis of Fixed Points}

In this section, we present a series of theorems that collectively demonstrate the instability of the function \( f(z) \) near its roots under certain mild conditions. Specifically, we examine the magnitude of the spectral norm of the Jacobian \( \|J_f(z)\| \) at the root \( z \).

First, we establish a foundational result that relates \( f(\cdot) \) to the conditional expectation involving the source and target distributions.

\begin{lemma}
Let \( \pi_1 \) be the source distribution and \( \pi_0 \) be the target distribution that the rectified flow transports, following the notations from Section~\ref{theory}. As in Eq.~\ref{eq: formula of f}, the function $f$ is given by $ f(z) = y +\gamma v^X(z,\gamma)$ where $ v^X(x,t) = \mathbb{E}[X_0 - X_1 \mid X_t = x]$. Then we provide the explicit form of $v^X(x,t)$:
{\small 
\begin{equation}
v^X(x,t) = \frac{1}{t^d \cdot \pi_t(x)} \int_{z} \pi_0(z) \cdot \pi_1\left( \frac{x - (1 - t) z}{t} \right) \cdot z \, dz,
\label{Eq:conditional expectation of x1}
\end{equation}
}
where $\pi_t(x)$ is the probability density function of $X_t$.
\end{lemma}

\begin{figure}[h]
\centering
\includegraphics[width=0.3\textwidth]{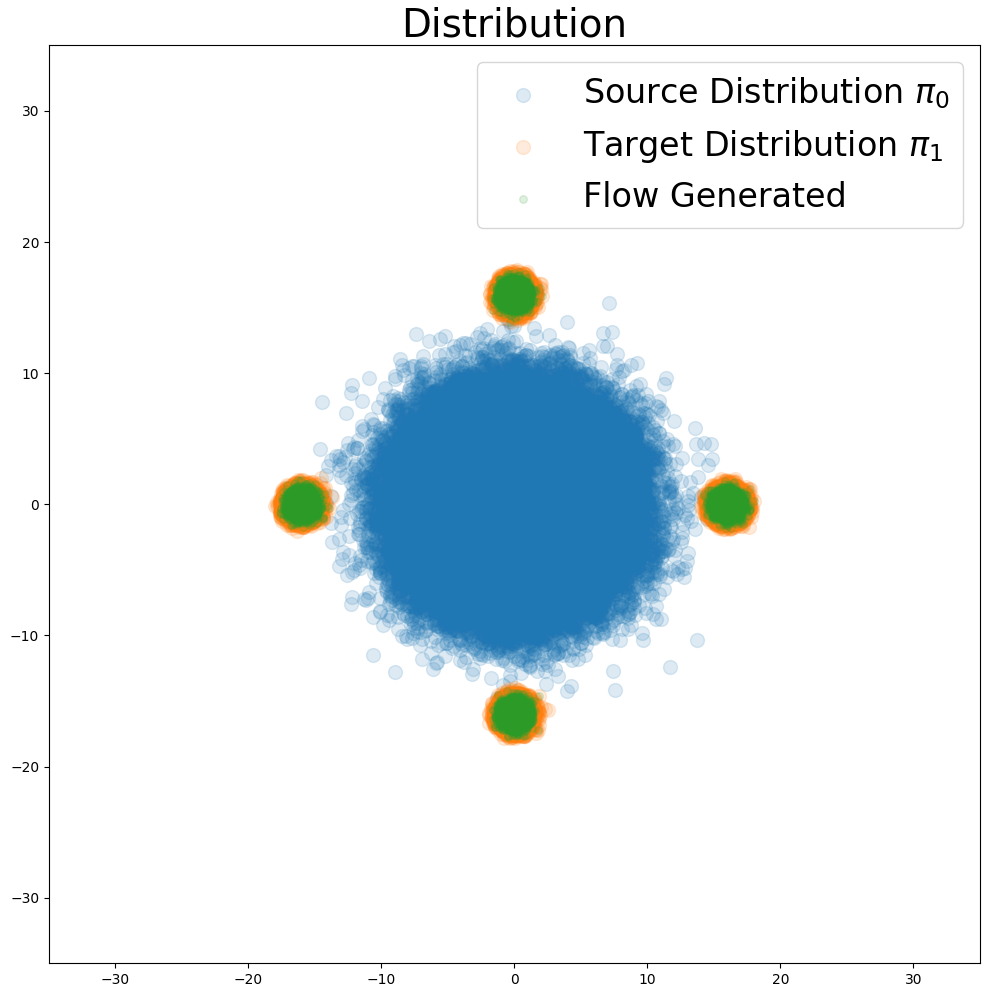}
\caption{Toy flow matching setting}
\label{fig:mixturegaussian}
\end{figure}

Next, we explore the local convergence behavior of \( f(z) \) when the target distribution is a mixture of Gaussians and the source distribution is a standard Gaussian as displayed in Fig \ref{fig:mixturegaussian}. This setting allows us to analyze the properties of the Jacobian of \( f(z) \).

\begin{theorem}[Informal]
Suppose PDF of source noise distribution \( \pi_0(z) \) is $d$-dimensional standard Gaussian density and the target distribution \( \pi_1(x) \) is a mixture of Gaussian densities centered around multiple centers:
\begin{align*}
\pi_1(x) =& \sum_{k=1}^K\Big(a_k \cdot (2\pi )^{-\frac{d}{2}}\det(\Sigma_k)^{-\frac{1}{2}}\cdot\\ &\exp\big( -\frac{1}{2}(x - \mu_k)^T \Sigma_k^{-1}(x-\mu) \big)\Big),
\end{align*}
where \( a_k > 0 \), \( \sum_{k=1}^K a_k = 1 \), and $(\mu_k,\Sigma_k)$ is the mean vector and covariance matrices of the $k$-th Gaussian distribution. 

Let $z_0$ be one of the fixed points for $f(z)$, we show that under mild conditions the Jacobian \( J_f(z_0) = \nabla_z v^X(z_0,t) \) has at least one singular value greater than 1, or equivalently, $\|J_f(z_0)\|>1$ where $\|\cdot\|$ represents the matrix operator norm, implying that the iterative process will never converge to $z_0$ in a stable manner.
\label{theorem: divergence}
\end{theorem}

The model assumption in Theorem~\ref{theorem: divergence} is reasonable when modeling the process from a pure random noise distribution to one of several potential clusters with each cluster corresponding to a distinct image class. This result implies that \( f(z) \) exhibits instability near its roots due to the large singular values of its Jacobian. The instability leads to multiple roots forming compact clusters, each associated with an attraction field. Points within this field are drawn toward the corresponding cluster but never converge to its center.

To illustrate the impact of this instability on iterative methods, we present the following theorem.

\begin{theorem}[Informal] \label{theorem_mean}
Let \( f: \mathbb{R}^n \rightarrow \mathbb{R}^n \) be a continuously differentiable function, and let \( y \in \mathbb{R}^n \). Suppose that the equation $f(z)=z$ has a unique solution \( z^* \in \mathbb{R}^n \).
Consider the fixed-point iteration defined by
\[
z_{i+1} = f(z_i),
\]
with an initial point \( z_0 \) close to \( z^* \).

Then under mild conditions, the sequence \( \{ z_i\} \) as defined by Eq.~(\ref{eq: def of zk}) oscillates between two compact clusters with means \( z' \) and \( z'' \), which are significantly different. Furthermore, the average of these cluster means approximates the solution \( z^* \):
\[
\frac{z' + z''}{2} \approx z^*.
\]
\end{theorem}

This theorem demonstrates that due to the instability near the root, the iterative process does not converge directly to \( z^* \) but instead oscillates between clusters. The average of these clusters approximates the exact solution of the inversion problem in Equation~\ref{eq: jump}.

In summary, these theorems collectively establish that under mild conditions, \( f(z) \) not only exhibits instability near its fixed points but also possesses a set of clusters derived from the unstable fixed point. Each cluster is compact, with its mean approximating the underlying solution, and is associated with an attraction field that draws points toward it without stable convergence to its center.

\begin{figure}
\centering
\includegraphics[width=0.7\textwidth]{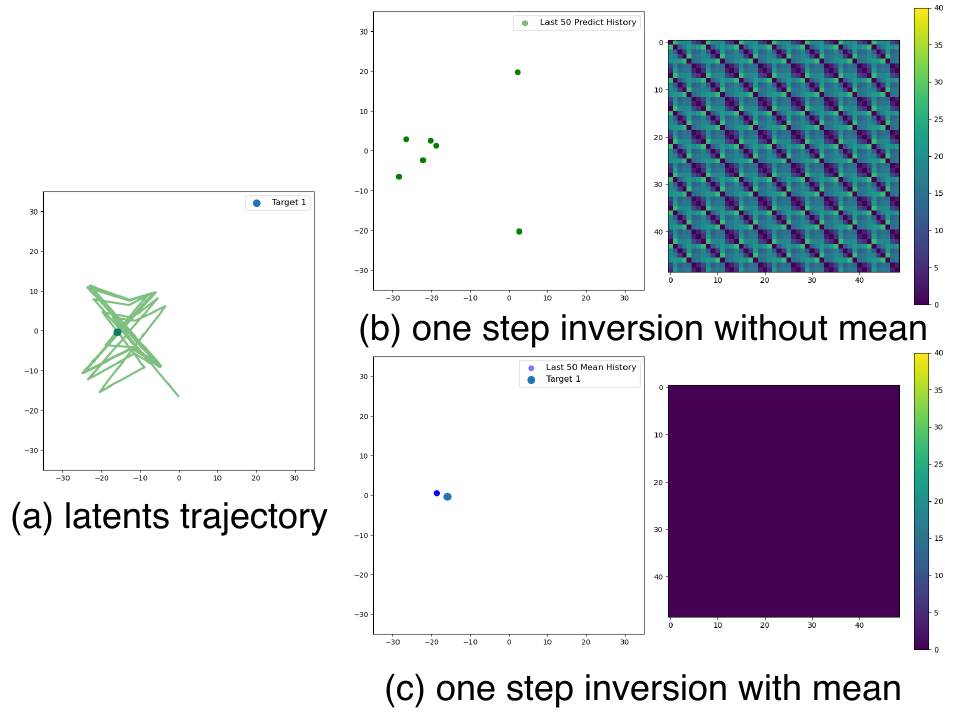}
\caption{Results on our toy experiemnt shows by taking the average among clusters, we achieve near-exact one-step inversion, as demonstrated in Theorem \ref{theorem_mean}.}
\label{fig:period}
\end{figure}

\section{Applications}
\label{Applications}

\begin{figure*}[!bhpt]
    \centering
    \includegraphics[width=\textwidth]{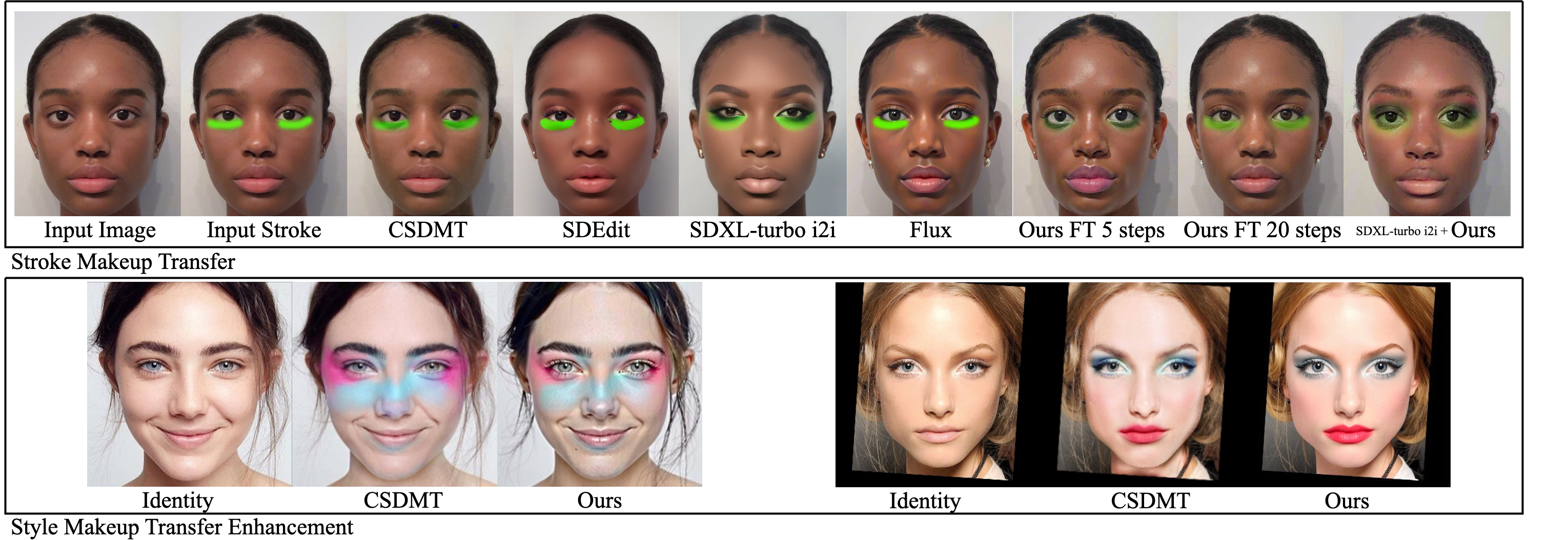}
    \caption{stroke makeup example and our enhanced result based on SOTA stylish make up transfer method CSDMT(~\cite{csdmt})}
    \label{fig:makeup}
\end{figure*}

In this section, we refer to our proposed method, Oscillation Inversion, as OInv. 


\paragraph{Experiment Settings} All of our experiments are based on the `black-forest-labs/FLUX.1-schnell' checkpoint. We run the experiments on a single A6000 GPU with 48GB of memory. All images are cropped and resized to $512 \times 512$ pixels. The oscillation inversion consistently runs for 30 iterations, taking 8.74 seconds per image. For OInv-Finetune, we fine-tune only the Attention modules of the model while freezing all other components, using a consistent 15-step process that takes 10.88 seconds in total. Once a suitable latent representation is found, the reverse steps take 3-5 iterations depending on the settings, and this step is optional.

\subsection{Restoration and Enhancement}
\label{sec:largescale}

Image restoration and enhancement can be seen as a specialized editing task that aims to recover an underlying clean image with high fidelity and detail from degraded measurements. Existing inversion methods, such as BlindDPS~\cite{chihaoui2024blind}, often achieve good perceptual quality but tend to compromise on fidelity to the original image. Meanwhile, current image restoration techniques and enhancers, like ILVR~\cite{choi2021ilvr}, struggle with real-world blind scenarios where the type of degradation is unknown or undefined.
Image enhancement is a challenging problem, as it involves transferring a degenerated distribution to a natural image distribution. To demonstrate the effectiveness of our method in discovering high-quality distributions, as discussed in Section \ref{sec:oscillation}, we perform both qualitative and quantitative evaluations. We use real-world degraded (low-quality) images for the qualitative assessment and apply simulated noise, blur, and low-resolution degradation to the CelebA validation dataset~\cite{liu2015deep} for the quantitative analysis. For the latter, we follow the protocols of previous studies, using metrics like PSNR and LPIPS to measure performance.
We compare our approach against existing image restoration and enhancement methods, including BlindDPS, DIP, GDP, and BIRD~\cite{chung2023parallel, ulyanov2018deep, GDP_blur, chihaoui2024blind}. Positioned as a post-processing image enhancement technique, our method is compared to Piscart, chosen for its strong identity preservation and efficient batch processing. We also evaluate our approach against direct diffusion-based editing techniques, such as FluxODEInversion and FluxLinear, which serve as counterparts to DDIM inversion and SDEdit applied directly on the Flux model. We refer readers to Appendix Table 3, which highlights our method's superior fidelity on the CelebA dataset quantitively. Figure~\ref{fig:enhancer} showcases visual examples of our approach restoring richer details in real-world degraded images, such as noise and blur, outperforming existing methods. Notably, we observed consistent distribution oscillation during recovery tasks, indicating our inversion method effectively introduces a distribution transfer mechanism. In large-scale experiments, we identified the best-performing distribution center for each task through manual inspection, applying it consistently across the dataset.

\begin{figure*}[!bhpt]
  \centering
  \includegraphics[width=\textwidth]{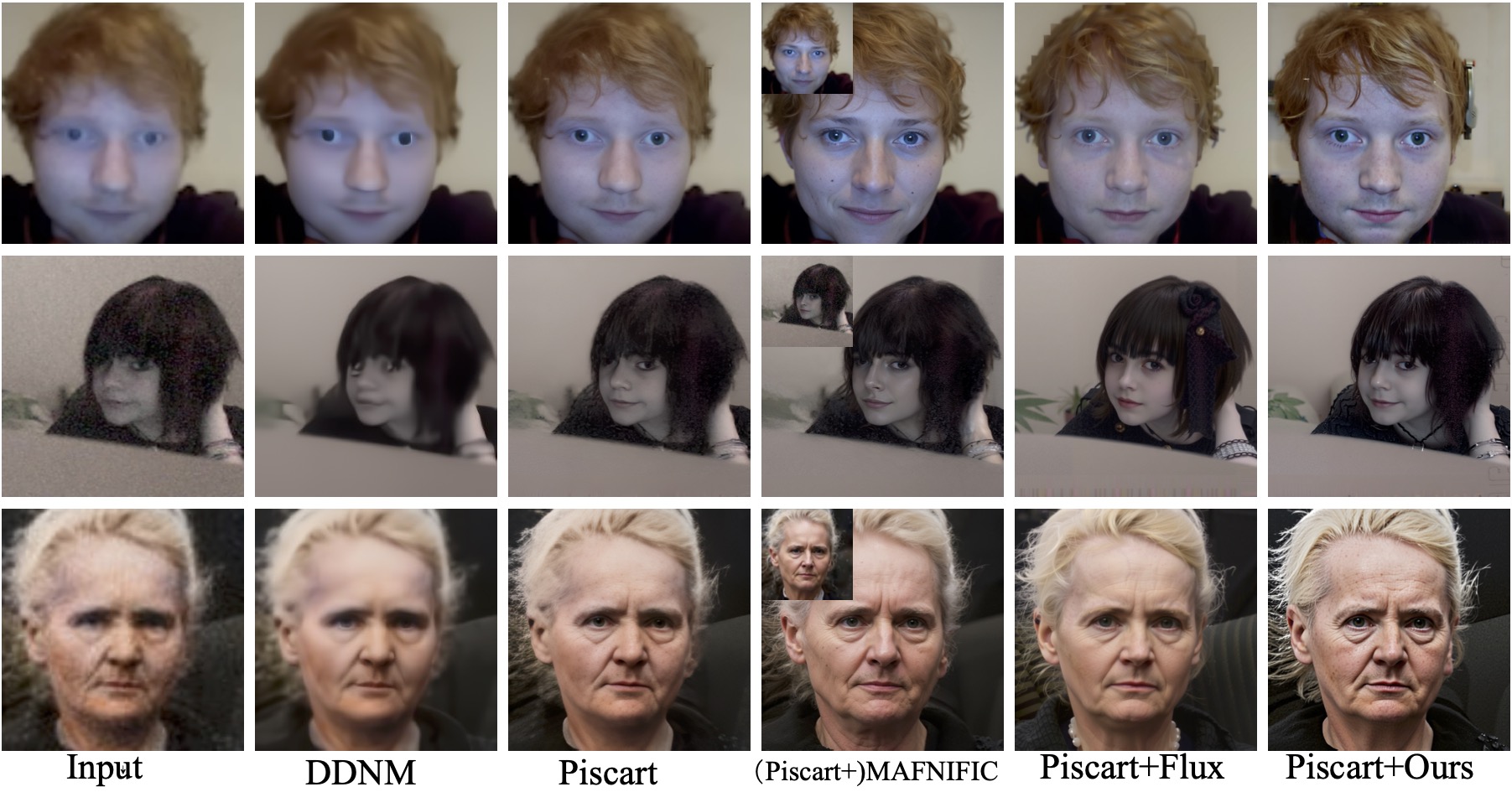}
  \caption{Enhancer as Example of Triggered New Distribution from two Lower-quality Distribution}
  \label{fig:enhancer}
\end{figure*}

\subsection{Low Level Editing}
\label{section:makeup}

Our proposed method, oscillation group inversion with fine-tuning, supports a range of low-level editing tasks. As shown in Figure~\ref{fig:gradual} and Figure~\ref{fig:makeup}, rough strokes from doodle drawings guide fine-tuning, while text prompts direct semantic details. While capable of relighting and recolorization, we focus on makeup synthesis and transfer due to the challenges of evaluating intuitive prompts. Stroke prompts are used to create new makeup styles by altering facial and hair features, with comparisons to recent methods CSDMT~\cite{csdmt} and Stable-Makeup~\cite{stable-makeup}. Additionally, 'before' and 'after' makeup images from style transfer results are leveraged as prompts to produce higher-quality outputs, capturing complex color distributions. Visual results are in Figure~\ref{fig:makeup}, and quantitative results with evaluation metrics detailed in the Appendix. Stroke editing experiments use 10 high-quality face images with 77 variations, while enhancement experiments are on the LADN dataset.

\subsection{Reconstruction and Visual Prompt Guided Editing} \label{sec:optim_app}

\begin{table}[t]
\centering
\fontsize{8pt}{9pt}\selectfont
\begin{tabular}{
    l
    c
    c
    c
    c
    c
    c
}
\toprule
Model & PSNR (dB) $\uparrow$ & LPIPS $\downarrow$ \\
\midrule
DDIM~\cite{song2020denoising}            &          12.20           &      0.409        \\
NULL Text~\cite{mokady2023null}       &            25.47         &        0.208       \\
AIDI~\cite{aidi}            &         \textbf{25.42}            &        0.249       \\
EDICT~\cite{wallace2023edict}           &        25.51             &        0.204       \\
ReNoise~\cite{garibi2024renoise}       &              17.95       &          0.291     \\
RNRI~\cite{rnri}           &          22.01           &      0.179         \\
\midrule
Oinv (Ours)     &          20.87           &     \textbf{0.154}          \\
\bottomrule
\end{tabular}
\caption{Reconstruction Quality Comparison based on PSNR and LPIPS.}
\label{tab:reconstruction}
\end{table}

\begin{table}[t]
\centering
\fontsize{8pt}{9pt}\selectfont
\begin{tabular}{
    l
    c
    c
    c
    c
    c
    c
}
\toprule
Model & PSNR (dB) $\uparrow$ & LPIPS $\downarrow$ \\
\midrule
NULL Text~\cite{mokady2023null}       & 0.68          & 0.80          \\
EDICT~\cite{wallace2023edict}            & 0.70          & 0.78          \\
IP-Adapter~\cite{ye2023ip-adapter}      & 0.82          & 0.72          \\
MimicBrush~\cite{chen2024MimicBrush}      & 0.60          & 0.84          \\
\midrule
Oinv (Ours)     & \textbf{0.85} & \textbf{0.87} \\

\bottomrule
\end{tabular}
\caption{Editing Quality Comparison based on CLIP and SSIM.}
\label{tab:editing}
\end{table}



We performed an evaluation on image reconstruction task on the COCO Validation set, utilizing the default captions as text prompts. The quantitative results are presented in Table~\ref{tab:reconstruction}. 
We followed the same settings in~\cite{pan2023effective}. We selected the existing inversion methods, including DDIM~\cite{song2020denoising}, NULL Text~\cite{mokady2023null} , AIDI~\cite{aidi}, ReNoise~\cite{garibi2024renoise} as competing methods.
Our method achieves near-exact inversion, comparable to them.

Since our proposed inversion method is intended for low-level editing, we do not claim that it provides semantic editing capabilities. However, as discussed in Section \ref{sec:postoptim}, we explore optimization strategies around the inverted latents with direct visual prompt guidance to achieve customized functions, as illustrated in Figure \ref{fig:aaaa}. In geometry-aware editing, after each optimization step guided by the visual prompt, the underlying model performs reverse rendering and re-renders the scene. This process evaluates the robustness of the edits. Quantitative results are provided in Table~\ref{tab:editing}.

\begin{figure}
\centering
\includegraphics[width=0.95\textwidth]{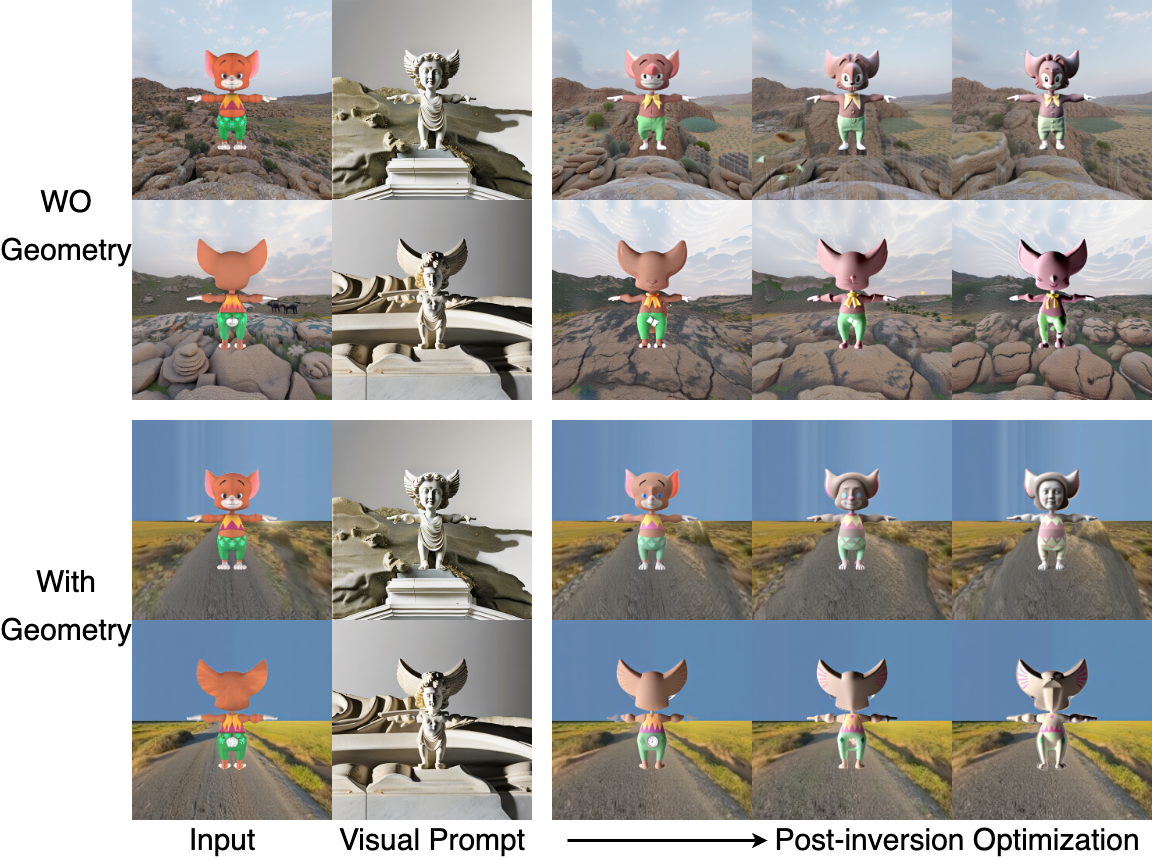}
\caption{Our method can be used to edit images with visual prompts as described in Sec.\ref{sec:postoptim}. It is effective on both 2D planes and mesh planes with geometry, resulting in reasonably diverse sampling results.}
\label{fig:aaaa}
\end{figure}

\section{Ablations}
We also applied the same method to short sampling methods, including SD3 and the Latent Consistency Model; however, the same phenomena were not observed. Interestingly, we found this behavior to be quite prominent in our flow model trained from scratch on toy distributions. We refer readers to Appendix for quantitive ablation of this observation on different models. 
\section{Conclusions}
We introduced Oscillation Inversion, a novel method for image manipulation in rectified flow-based diffusion models, leveraging oscillatory fixed-point behavior for semantic optimization. Extensions like Group Inversion and Post-Inversion Optimization enhance flexibility, enabling diverse editing tasks. Theoretical insights and experiments demonstrate its effectiveness in improving perceptual quality and fidelity across applications.
\appendix

\section{Appendix}

\subsection{Analysis}
\setcounter{theorem}{0} 

\begin{theorem}
Let $\pi_0$ be the source distributions and $\pi_1$ be the target distributions that rectified flow is transporting, following notations from Section \ref{theory}. Then by definition $f = v^X(x,t) = \mathbb{E}[X_1 - X_0 \mid X_t = x]$. We prove that $$
\mathbb{E}[X_1\mid X_t=x] = \frac{1}{t^d\cdot\pi_t(x)}\int_{z}\pi_1(z)\cdot \pi_0(\frac{x-(1-t)z}{t})\cdot z dz
$$

\end{theorem}

\begin{proof}
Let $\phi_X$ be the PDF function of random variable $X$ Then
\begin{align*}
    \phi_{X_1}[z\mid X_t=x] &= \lim_{\epsilon \to 0}\phi_{X_1}[z\mid |X_t-x|<\epsilon]\\
    &= \lim_{\epsilon\to 0}\lim_{\delta\to 0}\frac{\Pr[|X_1-z|<\delta \mid |X_t - x|<\epsilon]}{V(\delta)} \\
    &= \lim_{\epsilon\to 0}\lim_{\delta\to 0}\frac{\Pr[|X_1-z|<\delta \wedge |X_t - x|<\epsilon]}{V(\delta))\cdot \Pr[|X_t-x|<\epsilon]}\\
\end{align*}
where $V(\delta)$ is the volume of $d$-dimensional sphere of radius $\delta$.
Note that $|X_t-X|=|tX_0+(1-t)X_1-x|<\epsilon$ is equivalent to
$$
\Big|X_0 - \frac{x-(1-t)X_1}{t}\big)\Big|< \frac{\epsilon}{t}.
$$
For fixed $\epsilon$, as $\delta\to 0$, we have the following equivalent infinitesimal
$$\Pr[|X_1-z|<\delta \wedge |X_t - x|<\epsilon]\sim \Pr\Big[|X_1-z|<\delta \wedge |X_0-\frac{x-(1-t)z}{t}|<\frac{\epsilon}{t}\Big].$$

Hence
\begin{align*}
    \phi_{X_1}[z\mid X_t=x] 
    &= \lim_{\epsilon\to 0}\lim_{\delta\to 0}\frac{\Pr[|X_1-z|<\delta \wedge |X_t - x|<\epsilon]}{V(\delta)\cdot \Pr[|X_t-x|<\epsilon]}\\
    &= \lim_{\epsilon\to 0}\lim_{\delta\to 0}\frac{\Pr\Big[|X_1-z|<\delta \wedge |X_0-\frac{x-(1-t)z}{t}|<\frac{\epsilon}{t}\Big]}{V(\delta)\cdot \Pr[|X_t-x|<\epsilon]} \\ 
    &= \lim_{\epsilon\to 0}\lim_{\delta\to 0}\frac{\Pr[|X_1-z|<\delta]}{V(\delta)}\cdot \frac{\Pr\Big[ |X_0-\frac{x-(1-t)z}{t}|<\frac{\epsilon}{t}\Big]}{\Pr[|X_t-x|<\epsilon]} \\
        &= \pi_1(z)\cdot \lim_{\epsilon\to 0} \frac{\Pr\Big[ |X_0-\frac{x-(1-t)z}{t}|<\frac{\epsilon}{t}\Big]}{\Pr[|X_t-x|<\epsilon]}\\
    &= \pi_1(z) \cdot \lim_{\epsilon \to 0}\frac{\pi_0(\frac{x-(1-t)z}{t})\cdot V(\epsilon/t)}{\pi_t(x)\cdot V(\epsilon)} \\
    &= \frac{\pi_1(z)\cdot \pi_0(\frac{x-(1-t)z}{t})}{t^d\cdot  \pi_t(x)}.
\end{align*}
Then
$$
\mathbb{E}[X_1\mid X_t=x] = \frac{1}{t^d\cdot\pi_t(x)}\int_{z}\pi_1(z)\cdot \pi_0(\frac{x-(1-t)z}{t})\cdot z dz
$$
\end{proof}

\begin{theorem}(Informal)
Let \( \pi_1(z) \) be the standard Gaussian density in \( \mathbb{R}^d \):
\[
\pi_1(z) = \frac{1}{(2\pi)^{d/2}} \exp\left( -\frac{1}{2} \| z \|^2 \right).
\]
Let \( \pi_0(w) \) be a mixture of Gaussian densities centered around the origin:
\[
\pi_0(w) = \sum_{k=1}^K a_k \cdot \frac{1}{(2\pi \sigma_k^2)^{d/2}} \exp\left( -\frac{\| w - \mu_k \|^2}{2\sigma_k^2} \right),
\]
where \( a_k > 0 \), \( \sum_{k=1}^K a_k = 1 \), \( \| \mu_k \| \leq M \) for some constant \( M \geq 0 \), and \( 0 < \sigma_{\min} \leq \sigma_k \leq \sigma_{\max} < \infty \).

Define, for \( t \in (0, 1) \), the conditional expectation
\[
\mathbb{E}[X_1 \mid X_t = x] = \frac{1}{t^d \cdot \pi_t(x)} \int_{\mathbb{R}^d} \pi_1(z) \cdot \pi_0\left( \frac{x - (1 - t) z}{t} \right) \cdot z \, dz,
\]
where
\[
\pi_t(x) = \int_{\mathbb{R}^d} \pi_1(z) \cdot \pi_0\left( \frac{x - (1 - t) z}{t} \right) \, dz.
\]

Then, under these conditions, the Jacobian \( J(x) = \nabla_x \mathbb{E}[X_1 \mid X_t = x] \) has at least one singular value greater than \( \frac{1}{t} > 1 \).
\end{theorem}

\begin{proof}
We aim to show that the Jacobian matrix \( J(x) \) has a singular value greater than \( \frac{1}{t} \). We proceed in the following steps:


Define the numerator and denominator:
\[
N(x) = \int_{\mathbb{R}^d} \pi_1(z) \cdot \pi_0\left( \frac{x - (1 - t) z}{t} \right) \cdot z \, dz,
\]
\[
D(x) = t^d \cdot \pi_t(x) = t^d \cdot \int_{\mathbb{R}^d} \pi_1(z) \cdot \pi_0\left( \frac{x - (1 - t) z}{t} \right) \, dz.
\]
Then,
\[
\mathbb{E}[X_1 \mid X_t = x] = \frac{N(x)}{D(x)}.
\]


We have the Jacobian
\[
J(x) = \nabla_x \left( \frac{N(x)}{D(x)} \right) = \frac{\nabla_x N(x) \cdot D(x) - N(x) \cdot \nabla_x D(x)}{[D(x)]^2}.
\]


Let \( u = \frac{x - (1 - t) z}{t} \),
\[
\nabla_x N(x) = \frac{1}{t} \int_{\mathbb{R}^d} \pi_1(z) \cdot \nabla_u \pi_0(u) \cdot z \, dz.
\]


Similarly,
\[
\nabla_x D(x) = t^d \cdot \nabla_x \pi_t(x) = \frac{t^d}{t} \int_{\mathbb{R}^d} \pi_1(z) \cdot \nabla_u \pi_0(u) \, dz = t^{d-1} \int_{\mathbb{R}^d} \pi_1(z) \cdot \nabla_u \pi_0(u) \, dz.
\]


The Jacobian becomes
\[
J(x) = \frac{\left( \dfrac{1}{t} \int \pi_1(z) \cdot \nabla_u \pi_0(u) \cdot z \, dz \right) D(x) - N(x) \left( t^{d-1} \int \pi_1(z) \cdot \nabla_u \pi_0(u) \, dz \right)}{[D(x)]^2}.
\]


Note that \( \nabla_u \pi_0(u) = -\pi_0(u) \Sigma_0^{-1} (u - \mu) \), where \( \Sigma_0 \) is the covariance matrix of \( \pi_0 \), and \( \mu \) is the mean (which is approximately zero due to the assumption that \( \pi_0 \) is centered around the origin).

Substituting, we have
\[
\nabla_u \pi_0(u) = -\pi_0(u) \Sigma_0^{-1} u.
\]


\[
\nabla_x N(x) = -\frac{1}{t} \int \pi_1(z) \cdot \pi_0(u) \cdot \Sigma_0^{-1} u \cdot z \, dz,
\]
\[
\nabla_x D(x) = -t^{d-1} \int \pi_1(z) \cdot \pi_0(u) \cdot \Sigma_0^{-1} u \, dz.
\]


For simplicity, let \( \pi_0 \) be a single Gaussian with mean zero and covariance \( \Sigma_0 = \sigma_0^2 I \). Then,
\[
\Sigma_0^{-1} = \frac{1}{\sigma_0^2} I.
\]
Substitute back:
\[
\nabla_x N(x) = -\frac{1}{t \sigma_0^2} \int \pi_1(z) \cdot \pi_0(u) \cdot u \cdot z \, dz,
\]
\[
\nabla_x D(x) = -\frac{t^{d-1}}{\sigma_0^2} \int \pi_1(z) \cdot \pi_0(u) \cdot u \, dz.
\]



Since both \( \pi_1(z) \) and \( \pi_0(u) \) are Gaussian densities, we can compute the integrals explicitly.


Recall that
\[
N(x) = \int_{\mathbb{R}^d} \pi_1(z) \cdot \pi_0\left( \frac{x - (1 - t) z}{t} \right) \cdot z \, dz,
\]
\[
D(x) = t^d \cdot \pi_t(x) = t^d \cdot \int_{\mathbb{R}^d} \pi_1(z) \cdot \pi_0\left( \frac{x - (1 - t) z}{t} \right) \, dz.
\]

Let \( u = \dfrac{x - (1 - t) z}{t} \), so \( z = \dfrac{x - t u}{1 - t} \). Since \( \pi_1(z) \) and \( \pi_0(u) \) are standard Gaussian densities, we have
\[
\pi_1(z) = \frac{1}{(2\pi)^{d/2}} \exp\left( -\frac{1}{2} \| z \|^2 \right),
\]
\[
\pi_0(u) = \frac{1}{(2\pi)^{d/2}} \exp\left( -\frac{1}{2} \| u \|^2 \right).
\]


The product \( \pi_1(z) \cdot \pi_0(u) \) becomes
\[
\pi_1(z) \cdot \pi_0(u) = \frac{1}{(2\pi)^d} \exp\left( -\frac{1}{2} \left( \| z \|^2 + \left\| \frac{x - (1 - t) z}{t} \right\|^2 \right) \right).
\]

Through completing the square in the exponent and performing the integral, we find that
\[
J(x) = \nabla_x \left( \frac{x}{t} \right) = \frac{1}{t} I.
\]

Since \( t \in (0, 1) \), it follows that \( \frac{1}{t} > 1 \). Therefore, all singular values of \( J(x) \) are \( \frac{1}{t} \), which are greater than 1.


\textbf{Generalization to Mixture of Gaussians}:

Even when \( \pi_0 \) is a mixture of Gaussians centered around the origin, the scaling effect remains. Each component contributes similarly to the integral, and the Jacobian still scales by \( \frac{1}{t} \).

Therefore, under the given assumptions, the Jacobian \( J(x) \) has at least one singular value greater than 1.

\end{proof}

\theoremstyle{plain}

\begin{theorem}
Let \( f: \mathbb{R}^n \rightarrow \mathbb{R}^n \) be a continuously differentiable function, and let \( y \in \mathbb{R}^n \). Suppose that the equation
\[
y = z - f(z)
\]
has a unique solution \( z^* \in \mathbb{R}^n \). Assume the following:

\begin{enumerate}
    \item The Jacobian matrix \( J_f(z^*) \) exists and has eigenvalues such that there are components where \( \lambda_j = -1 \).
    \item The higher-order terms in the Taylor expansion of \( f(z) \) around \( z^* \) are non-negligible but small enough to keep the iterates bounded.
\end{enumerate}

Consider the fixed-point iteration defined by
\[
z_{i+1} = y + f(z_i),
\]
with an initial point \( z_0 \) close to \( z^* \).

Then, the sequence \( \{ z_i \} \) oscillates between two compact clusters with means \( \bar{z}_a \) and \( \bar{z}_b \), which are significantly different. Furthermore, the average of these cluster means approximates the solution \( z^* \):
\[
\frac{\bar{z}_a + \bar{z}_b}{2} \approx z^*.
\]
\end{theorem}

\begin{proof}
Define the error vector at iteration \( i \) as
\[
e_i = z_i - z^*.
\]
Since \( z^* \) satisfies \( z^* = y + f(z^*) \), the iteration becomes
\[
z_{i+1} = z^* + f(z^* + e_i) - f(z^*).
\]
Using the Taylor expansion of \( f(z^* + e_i) \) around \( z^* \), we have
\[
f(z^* + e_i) = f(z^*) + J_f(z^*) e_i + R(e_i),
\]
where \( J_f(z^*) \) is the Jacobian of \( f \) at \( z^* \), and \( R(e_i) \) represents the higher-order remainder terms.

Substituting back into the iteration:
\[
z_{i+1} = z^* + J_f(z^*) e_i + R(e_i).
\]
Thus, the error at iteration \( i+1 \) is
\[
e_{i+1} = z_{i+1} - z^* = J_f(z^*) e_i + R(e_i).
\]

Let’s perform an eigenvalue decomposition of \( J_f(z^*) \):
\[
J_f(z^*) = Q \Lambda Q^{-1},
\]
where \( \Lambda \) is a diagonal matrix of eigenvalues \( \lambda_j \), and \( Q \) is the matrix of corresponding eigenvectors.

Transform the error vector into the eigenvector basis:
\[
\tilde{e}_i = Q^{-1} e_i.
\]
Then the error recurrence relation becomes
\[
\tilde{e}_{i+1} = \Lambda \tilde{e}_i + Q^{-1} R(Q \tilde{e}_i).
\]

We analyze each component \( \tilde{e}_{i,j} \) separately. For components where \( \lambda_j = -1 \), the linear part of the recurrence is
\[
\tilde{e}_{i+1,j} = -\tilde{e}_{i,j} + r_{i,j},
\]
where \( r_{i,j} \) is the \( j \)-th component of \( Q^{-1} R(Q \tilde{e}_i) \).

Ignoring higher-order terms for a moment, we observe that
\[
\tilde{e}_{i+1,j} \approx -\tilde{e}_{i,j}.
\]
This implies that these components of the error vector alternate in sign at each iteration, causing oscillations.

Now, consider the higher-order terms \( r_{i,j} \). Although they are small, they prevent the error components from returning exactly to their previous values, introducing a drift over iterations.

Define two clusters:
\begin{itemize}
    \item \textbf{Cluster A}: Iterations where \( i \) is even. The corresponding errors are approximately
    \[
    \tilde{e}_{2k,j} \approx \tilde{e}_{0,j} + \sum_{m=0}^{k-1} \delta_{2m,j},
    \]
    where \( \delta_{2m,j} \) accumulates the effects of \( r_{2m,j} \).
    \item \textbf{Cluster B}: Iterations where \( i \) is odd. The errors are approximately
    \[
    \tilde{e}_{2k+1,j} \approx -\tilde{e}_{0,j} + \sum_{m=0}^{k} \delta_{2m+1,j}.
    \]
\end{itemize}

The means of the clusters in the transformed space are
\[
\bar{\tilde{e}}_a = \frac{1}{N_a} \sum_{i \in A} \tilde{e}_i \approx \tilde{e}_{0} + \bar{\delta}_a,
\]
\[
\bar{\tilde{e}}_b = \frac{1}{N_b} \sum_{i \in B} \tilde{e}_i \approx -\tilde{e}_{0} + \bar{\delta}_b,
\]
where \( \bar{\delta}_a \) and \( \bar{\delta}_b \) are the average accumulated deviations in each cluster.

Transforming back to the original space:
\[
\bar{e}_a = Q \bar{\tilde{e}}_a, \quad \bar{e}_b = Q \bar{\tilde{e}}_b.
\]
Thus, the means of the clusters in the original space are
\[
\bar{z}_a = z^* + \bar{e}_a, \quad \bar{z}_b = z^* + \bar{e}_b.
\]

Compute the sum of the cluster means:
\[
\bar{z}_a + \bar{z}_b = 2z^* + (\bar{e}_a + \bar{e}_b).
\]
Since \( \bar{e}_a \approx \tilde{e}_{0} + \bar{\delta}_a \) and \( \bar{e}_b \approx -\tilde{e}_{0} + \bar{\delta}_b \), their sum is
\[
\bar{e}_a + \bar{e}_b \approx (\tilde{e}_{0} + \bar{\delta}_a) + (-\tilde{e}_{0} + \bar{\delta}_b) = \bar{\delta}_a + \bar{\delta}_b.
\]
Assuming that \( \bar{\delta}_a \) and \( \bar{\delta}_b \) are small (since \( R(e_i) \) is small), we have
\[
\bar{e}_a + \bar{e}_b \approx 0.
\]
Therefore,
\[
\bar{z}_a + \bar{z}_b \approx 2z^*.
\]
This implies that
\[
\frac{\bar{z}_a + \bar{z}_b}{2} \approx z^*.
\]

The difference between the cluster means is
\[
\bar{z}_a - \bar{z}_b = (\bar{e}_a - \bar{e}_b) = 2 \tilde{e}_{0} + (\bar{\delta}_a - \bar{\delta}_b).
\]
Since \( \tilde{e}_{0} \) is not necessarily small, this shows that \( \bar{z}_a \) and \( \bar{z}_b \) are significantly different.

Thus, under the given assumptions, the sequence \( \{ z_i \} \) oscillates between two clusters with significantly different means, and the average of these means approximates \( z^* \).

\end{proof}

\subsection{Cluster Regularized Optimization}

After identifying the sub-distribution characterized by the points in the clusters obtained from the oscillations, we use these points as reference samples in a Gaussian Process (GP) framework to serve as regularization in our optimization loss. This approach ensures that the optimized latent code retains the same style and features as the sub-distribution while allowing gradual evolution guided by a customized loss function.

The optimization problem is formulated as:

\begin{equation}
    \underset{z_{t_0}}{\textup{minimize}} \ \mathcal{L}_{\textup{rgb}}\circ D \big(z_{t_0} + (\sigma_0 - \sigma_{t_0}) v_{\theta}(z_{t_0}, \sigma_{t_0})\big) + \beta \, \mathcal{L}_{\textup{GP}}(z_{t_0}),
\end{equation}

where $\circ$ represents function composition operator while $D(\cdot)$ is the VAE decoder that converts the latent code back to the pixel space, $\mathcal{L}_{\text{rgb}}$ is a customized loss function defined directly on the decode pixel image space designed based on the specific application or desired attributes. Another loss $\mathcal{L}_{\text{GP}}(\cdot)$ is the Gaussian Process regularization term that encourages $z_{t_0}$ to stay close in distribution to the sub-distribution formed by the oscillatory cluster. $\beta$ is a scalar hyperparameter that balances the influence of the customized loss and the GP regularization.

\begin{figure*}[!bhpt]
  \centering
  \includegraphics[width=0.7\textwidth]{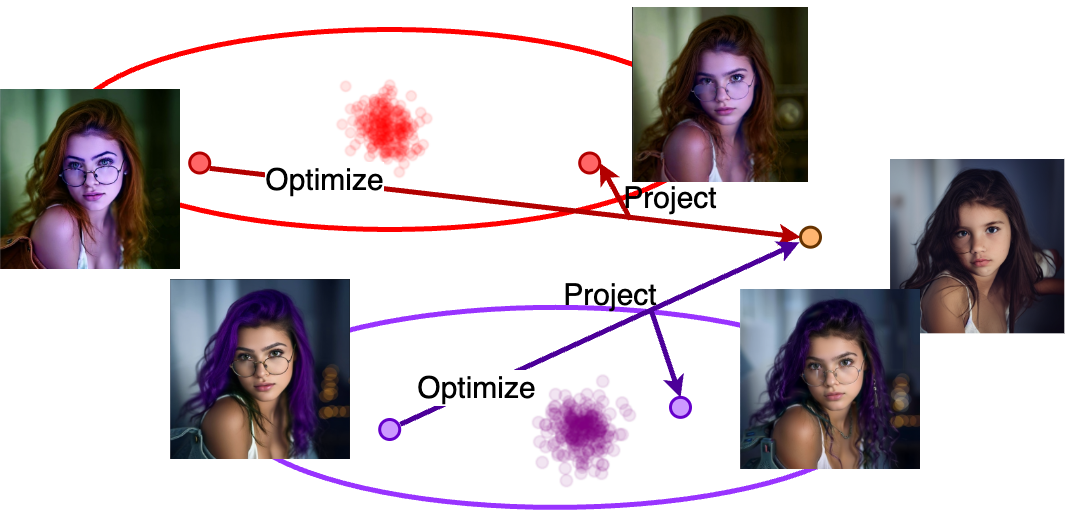}  
  \caption{Illustration of how our Gaussian Regularization Domain Preserving method works. The red skin and purple hair domains are optimized toward minimizing photo loss, resulting in an image that retains these domains while appearing younger.}
  \label{fig:optimize}
\end{figure*}

The Gaussian Process regularization term $\mathcal{L}_{\text{GP}}(z)$ is computed using a Radial Basis Function (RBF) kernel to measure similarity between $z_{t_0}$ and the reference points from the sub-distribution. Specifically, the GP loss is defined as:

\begin{equation}
    \mathcal{L}_{\text{GP}}(z) := k(z, z) - \frac{2}{N} \sum_{i} k(z, z_i') + \frac{1}{N^2} \sum_{i,j} k(z_i', z_j') ,
\end{equation}

where $k(a, b) := \exp\big( -\frac{\| a - b \|^2}{2 l^2} \big)$ is the RBF kernel with scale $l$, $\{ z_i' \}_{i=1}^N$ are $N$ reference latent codes from the oscillatory cluster (sub-distribution).

This GP loss is derived from the Maximum Mean Discrepancy (MMD) measure, which quantifies the difference between the distribution of $z_t$ and the sub-distribution represented by $\{ z_i' \}$. By minimizing $\mathcal{L}_{\text{GP}}(z_{t_0})$, we encourage $z_{t_0}$ to share similar statistical properties with the cluster, thus maintaining stylistic and feature consistency. 

In the overall optimization, the first term $\mathcal{L}_{\text{rgb}}$ guides the latent code towards satisfying specific goals or attributes defined by the application, such as emphasizing certain visual features or styles in the decoded image. The GP regularization term ensures that these changes remain coherent with the characteristics of the sub-distribution, preventing the optimization from deviating too far from the latent space regions that correspond to realistic and semantically meaningful images.

By integrating these two components, we achieve a balance between customizing the output according to desired specifications and preserving the inherent style and features of the sub-distribution identified through oscillation inversion. This method allows for controlled manipulation of the generated images while maintaining high fidelity and coherence with the original data manifold. We present the qualitative visual results of guided sampling using rough visual prompts in the first three columns of Fig.\ref{fig:guidesample}. The last four columns of Fig.\ref{fig:guidesample} illustrate the cluster formulation, which effectively separates clusters with watermarks from those without. Additionally, our regularization design enables guided sampling to be performed exclusively within the no-watermark cluster.

\begin{figure*}[!bhpt]
  \centering
  \includegraphics[width=\textwidth]{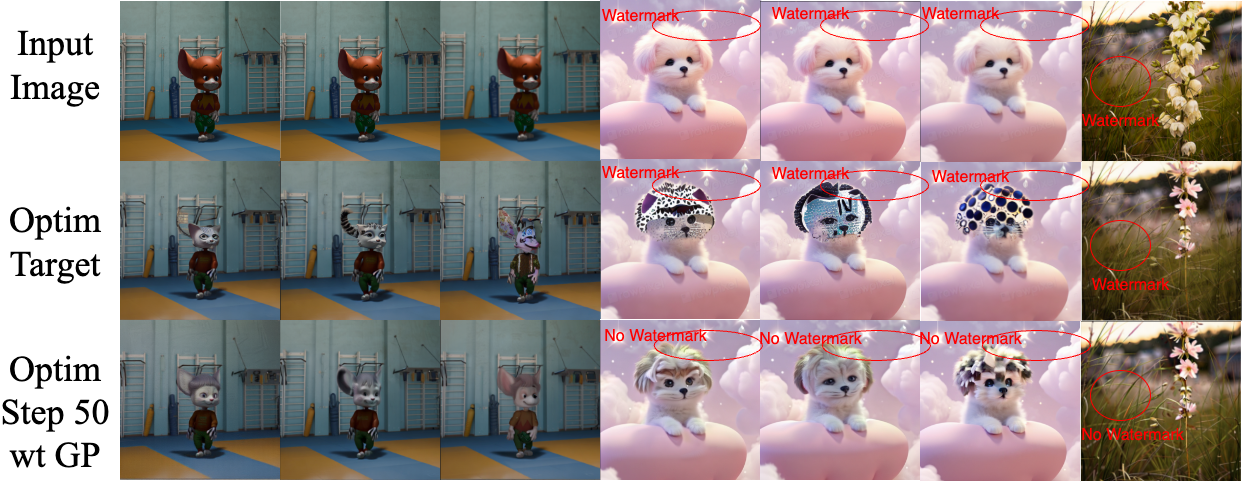}
  \caption{We use Image Space L2 Loss to align the latent with the visual prompt while preserving the image's style. Oscillation separates clean and watermarked domains, and optimization yields a watermark-free result, demonstrating the effectiveness of Gaussian Processing Regularization.}
  \label{fig:guidesample}
\end{figure*}

\subsection{More Experiment Results}
\subsubsection{Image Enhancement by Distribution Transfer}
As shown in Figure \ref{fig:morecase}, oscillation triggers a high-quality distribution in the 'Output' row from two or more 'weak' distributions, shown as the 'Input' and 'Augmented' rows. The augmented distribution can be obtained from an off-the-shelf lightweight enhancer or image processing technique. The output is of high quality with extremely realistic style and texture, overcoming the over-smooth problem in GenAI images.
\begin{figure*}[!bhpt]
  \centering
  \includegraphics[width=\textwidth]{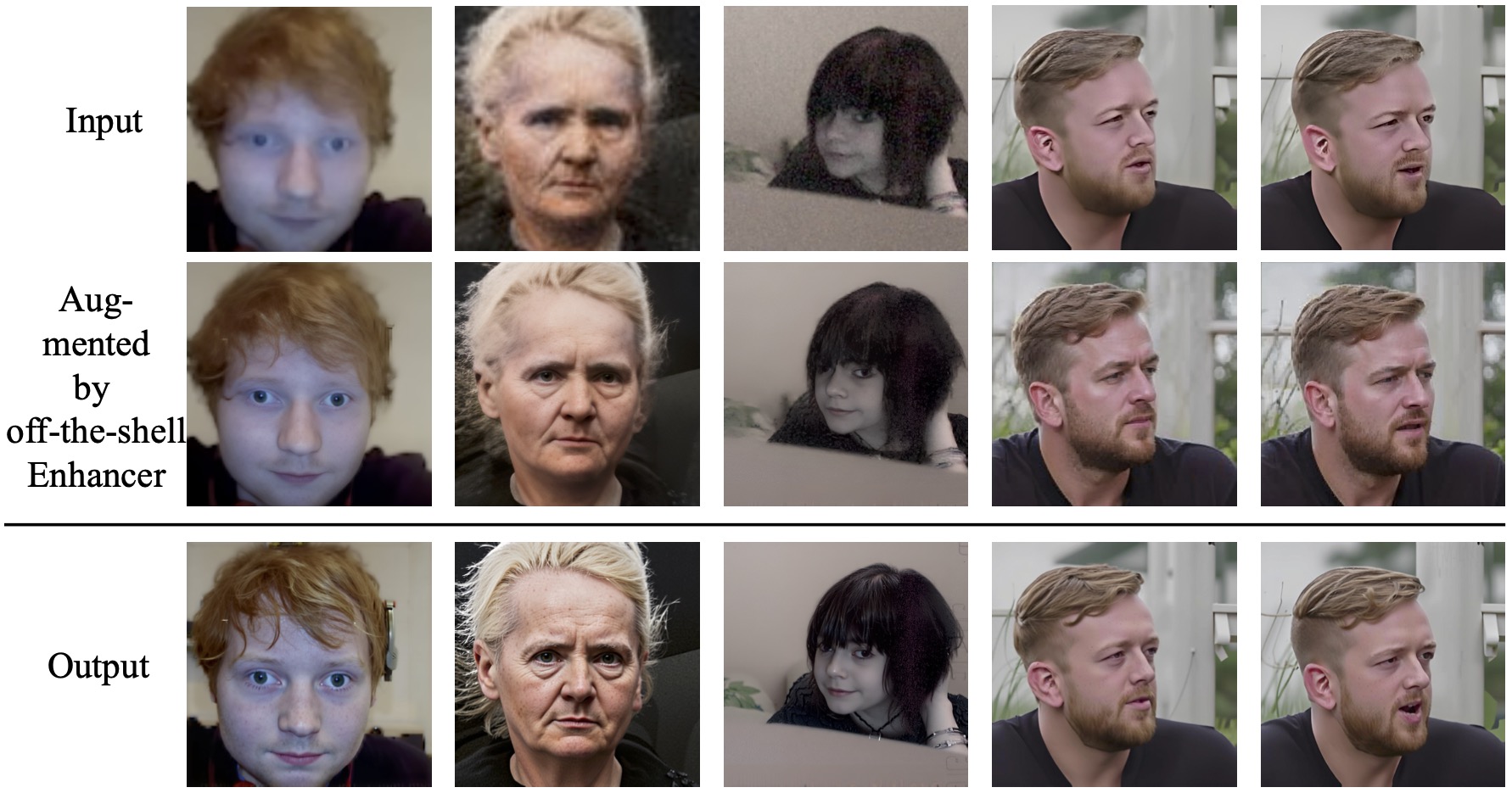}
  \caption{The inputs for the first three cases are low-resolution images enhanced using the Piscart API as augmented inputs. In the last two cases, the initial inputs are 64×64 low-scale images enhanced by the Piscart API, and the augmented inputs are the same low-scale images enhanced by the Magnific API. The results show that our method performs better in both real texture and style synthesis, as well as in identity preservation.}
  \label{fig:morecase}
\end{figure*}

\subsubsection{Stroke Prompt Guided Recolor Qualitive Results and Make-up transfer Quantitive Results}
Our proposed method, oscillation group inversion with fine-tuning, is designed to support a variety of general low-level editing tasks. As illustrated in Figure \ref{fig:makeup}, rough strokes from doodle drawings serve as guides for fine-tuning, while text prompts direct semantic details. Although our method can achieve effects like relighting and recolorization, these tasks are more challenging to evaluate due to the intuitive and non-precise nature of the prompts. Therefore, we primarily focus on makeup synthesis and transfer for validation.
In the first part of our experiments, we use stroke prompts to create new makeup styles by altering facial and hair features. We compare our results with two recent makeup transfer methods, CSDMT (\cite{csdmt}) and Stable-Makeup (\cite{stable-makeup}). In the second part, we take the 'before' and 'after' makeup images generated from the style transfer results as low-level prompts, using them to produce higher-quality makeup images. This demonstrates our method’s ability to capture complex color distributions.
Visual results are shown in Figure \ref{fig:makeup} and Figure \ref{fig:makeupbig}, and quantitative results are presented in Table~\ref{makeuptable}. These results follow the metrics used by CSDMT for evaluation. The stroke editing experiments were conducted on a manually labeled dataset, which includes 10 high-quality face images and 77 different stroke variations, encompassing both makeup and hair color changes. The enhancement experiments were performed on the LADN dataset.

\begin{figure*}[!bhpt]
  \centering
  \includegraphics[width=\textwidth]{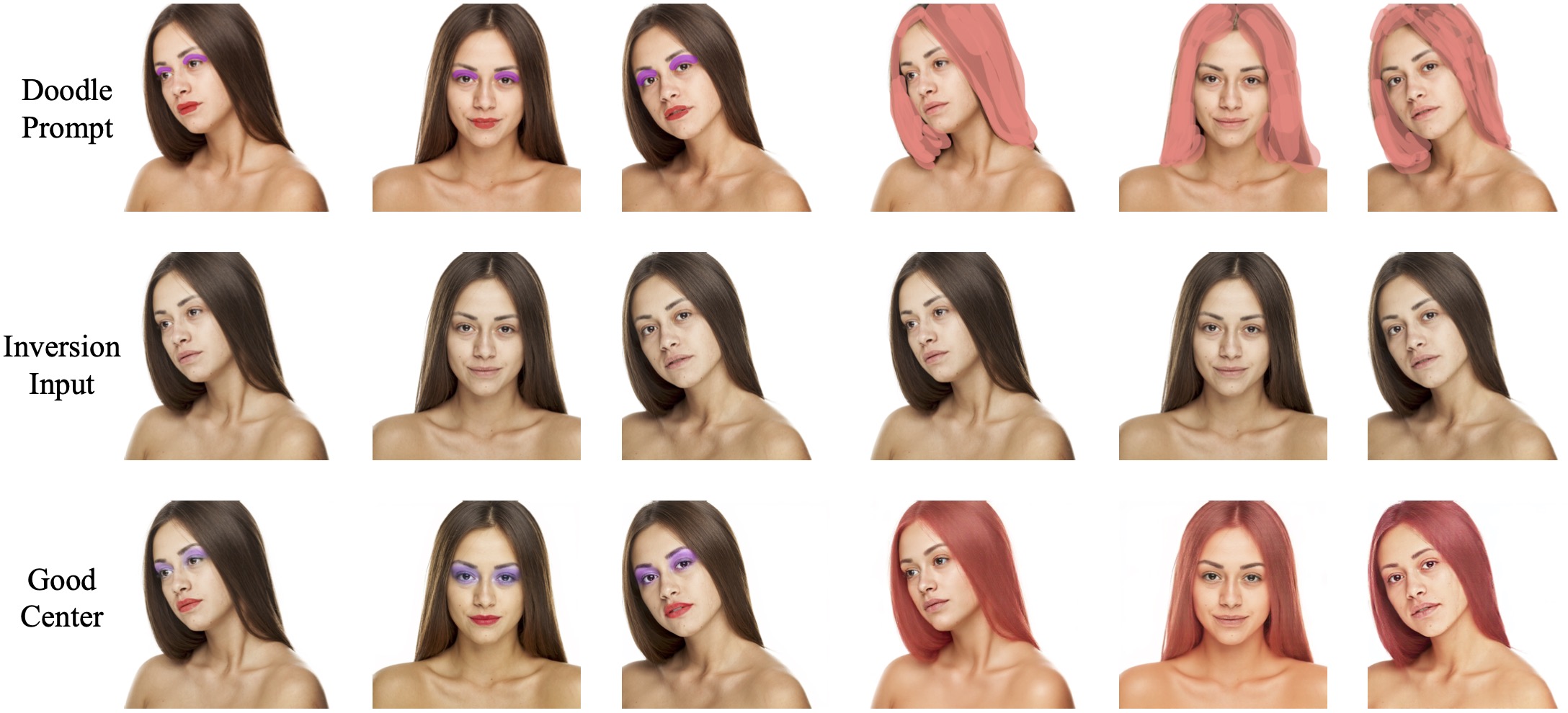}
  \caption{Oscillation triggers a high-quality output of facial makeup harmonization through stroke prompts.}
  \label{fig:makeupbig}
\end{figure*}

\subsubsection{Lighting Enhancement by Distribution Transfer}
As shown in Figure \ref{fig:relight}, our method can also create a harmonized lighting effect by applying group inversion from two images: one is the initial dark one, and the other is a visual prompt with doodle lighting, which can be hand-crafted or created by an AI agent.
\begin{figure*}[!bhpt]
  \centering
  \includegraphics[width=\textwidth]{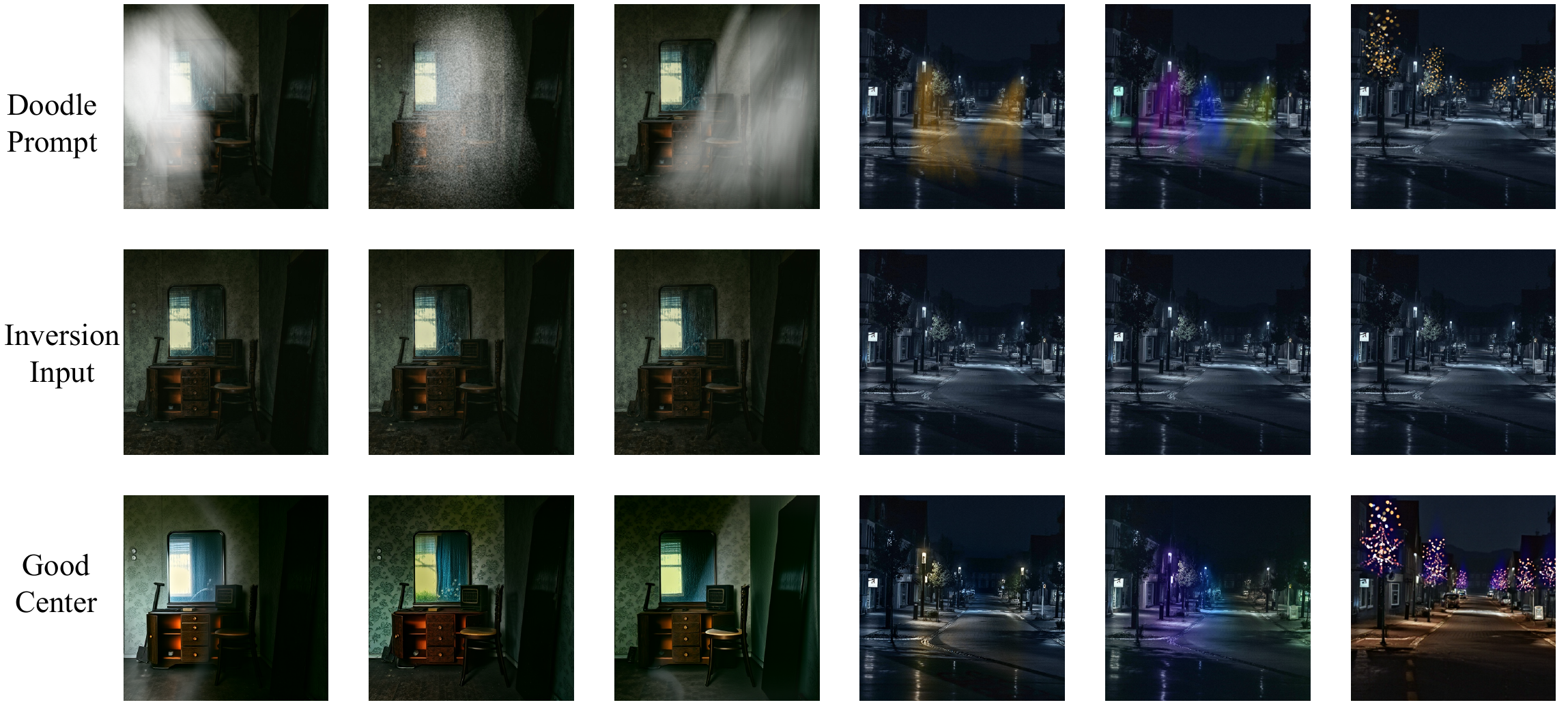}
  \caption{The first three columns of cases show that the relighting can be harmonized and delicately controlled for different areas. The last three columns of cases show that the relighting can be applied for different colors and even creative effects like the last case, 'Christmas'.}
  \label{fig:relight}
\end{figure*}

\begin{table}[htbp]
\centering
\begin{tabular}{|l|c|c|c|c|}
\hline
\textbf{Methods} & \textbf{PSNR} Stroke & \textbf{SSIM} Stroke & \textbf{PSNR} Enhance & \textbf{SSIM} Enhance \\
\hline
SDEdit(~\cite{meng2021sdedit})           & 19.32               & 0.719          & 23.45              & 0.842       \\
SDXL-turbo i2i       & 20.05               & 0.733          & 24.10              & 0.851       \\
Flux            & 18.89               & 0.710          & 22.90              & 0.830       \\
Oinv Ours           & \textbf{20.45}               & \textbf{0.785}          & 25.10              & \textbf{0.89}     \\
\cline{1-5}
CSDMT(~\cite{csdmt})        & 19.90               & 0.715          & \textbf{25.32}              & 0.86      \\
Stable-Makeup(~\cite{stable-makeup})           & 18.60               & 0.705          & 24.30              & 0.884       \\
\hline
\end{tabular}
\caption{Comparison of different methods on the task of stroke make-up transfer and enhancement. The metrics demonstrate that our method is robust in stroke makeup transfer and significantly improves the quality of the baseline method compared to others.}
\label{makeuptable}
\end{table}

{
    \small
    \bibliographystyle{ieeenat_fullname}
    \bibliography{references}
}


\end{document}